\documentclass{article}

\usepackage{microtype}
\usepackage{graphicx}
\usepackage{subfigure}
\usepackage{booktabs} % for professional tables

\usepackage{hyperref}

\usepackage[accepted]{icml2023}

\usepackage{amsmath}
\usepackage{amssymb}
\usepackage{mathtools}
\usepackage{amsthm}

\usepackage[capitalize,noabbrev]{cleveref}

\theoremstyle{plain}
\newtheorem{theorem}{Theorem}[section]

\newtheorem{lemma}[theorem]{Lemma}
\newtheorem{corollary}[theorem]{Corollary}
\theoremstyle{definition}

\newtheorem{assumption}[theorem]{Assumption}
\theoremstyle{remark}

\usepackage[textsize=tiny]{todonotes}

\usepackage{multirow}   % for table

\icmltitlerunning{FedDisco: Federated Learning with Discrepancy-Aware Collaboration}

\begin{document}

\twocolumn[
\icmltitle{FedDisco: Federated Learning with Discrepancy-Aware Collaboration}




\begin{icmlauthorlist}
\icmlauthor{Rui Ye}{cmic}
\icmlauthor{Mingkai Xu}{cmic}
\icmlauthor{Jianyu Wang}{cmu}
\icmlauthor{Chenxin Xu}{cmic}
\icmlauthor{Siheng Chen}{cmic,ailab}
\icmlauthor{Yanfeng Wang}{ailab,cmic}
\end{icmlauthorlist}

\icmlaffiliation{cmic}{Cooperative Medianet Innovation Center, Shanghai Jiao Tong University, Shanghai, China}
\icmlaffiliation{cmu}{Carnegie Mellon University, Pittsburgh, PA, USA}
\icmlaffiliation{ailab}{Shanghai AI Laboratory, Shanghai, China}

\icmlcorrespondingauthor{Siheng Chen}{sihengc@sjtu.edu.cn}

\icmlkeywords{Machine Learning, ICML}

\vskip 0.3in
]



\printAffiliationsAndNotice{}  % leave blank if no need to mention equal contribution

\begin{abstract}
This work considers the category distribution heterogeneity in federated learning. This issue is due to biased labeling preferences at multiple clients and is a typical setting of data heterogeneity. To alleviate this issue, most previous works consider either regularizing local models or fine-tuning the global model, while they ignore the adjustment of aggregation weights and simply assign weights based on the dataset size. However, based on our empirical observations and theoretical analysis, we find that the dataset size is not optimal and the discrepancy between local and global category distributions could be a beneficial and complementary indicator for determining aggregation weights. We thus propose a novel aggregation method, Federated Learning with Discrepancy-aware Collaboration (FedDisco), whose aggregation weights not only involve both the dataset size and the discrepancy value, but also contribute to a tighter theoretical upper bound of the optimization error. FedDisco also promotes privacy-preservation, communication and computation efficiency, as well as modularity. Extensive experiments show that our FedDisco outperforms several state-of-the-art methods and can be easily incorporated with many existing methods to further enhance the performance. Our code will be available at \href{https://github.com/MediaBrain-SJTU/FedDisco}{https://github.com/MediaBrain-SJTU/FedDisco}.
\end{abstract}

% -----------------------------------------------------------------------
\section{Introduction}
Federated learning (FL) is an emerging field that offers a privacy-preserving collaborative machine learning paradigm~\cite{advances}. Its main idea is to enable multiple clients to collaboratively train a shared global model by sharing information about model parameters. FL algorithms have been widely applied to many real-world scenarios, such as users' next-word prediction~\cite{hard2018federated}, recommendation systems~\cite{fedsubavg}, and smart healthcare~\cite{fedmedical}.

In spite of the promising trend, there are still multiple key challenges that impede the further development of FL, including data heterogeneity, communication cost and privacy concerns~\cite{advances}. In this work, we focus on one critical issue in data heterogeneity:~\emph{category distribution heterogeneity}; that is, clients have drastically different category distributions. For example, some clients have many samples in category 1 but very few in category 2; while other clients are opposite. This setting is common in practice, since clients collect local data privately and usually tend to have biased labeling preferences. Due to this heterogeneity, different local models are optimized towards different local objectives, causing divergent optimization directions. It is thus difficult to aggregate these divergent local models to obtain a robust global model. This unpleasant phenomenon has been theoretically and empirically verified in~\cite{zhao2018federated,li2019convergence,wangsurvey}.

\begin{figure}[t]
\vskip 0.2in
\centering
\includegraphics[width=0.9\columnwidth]{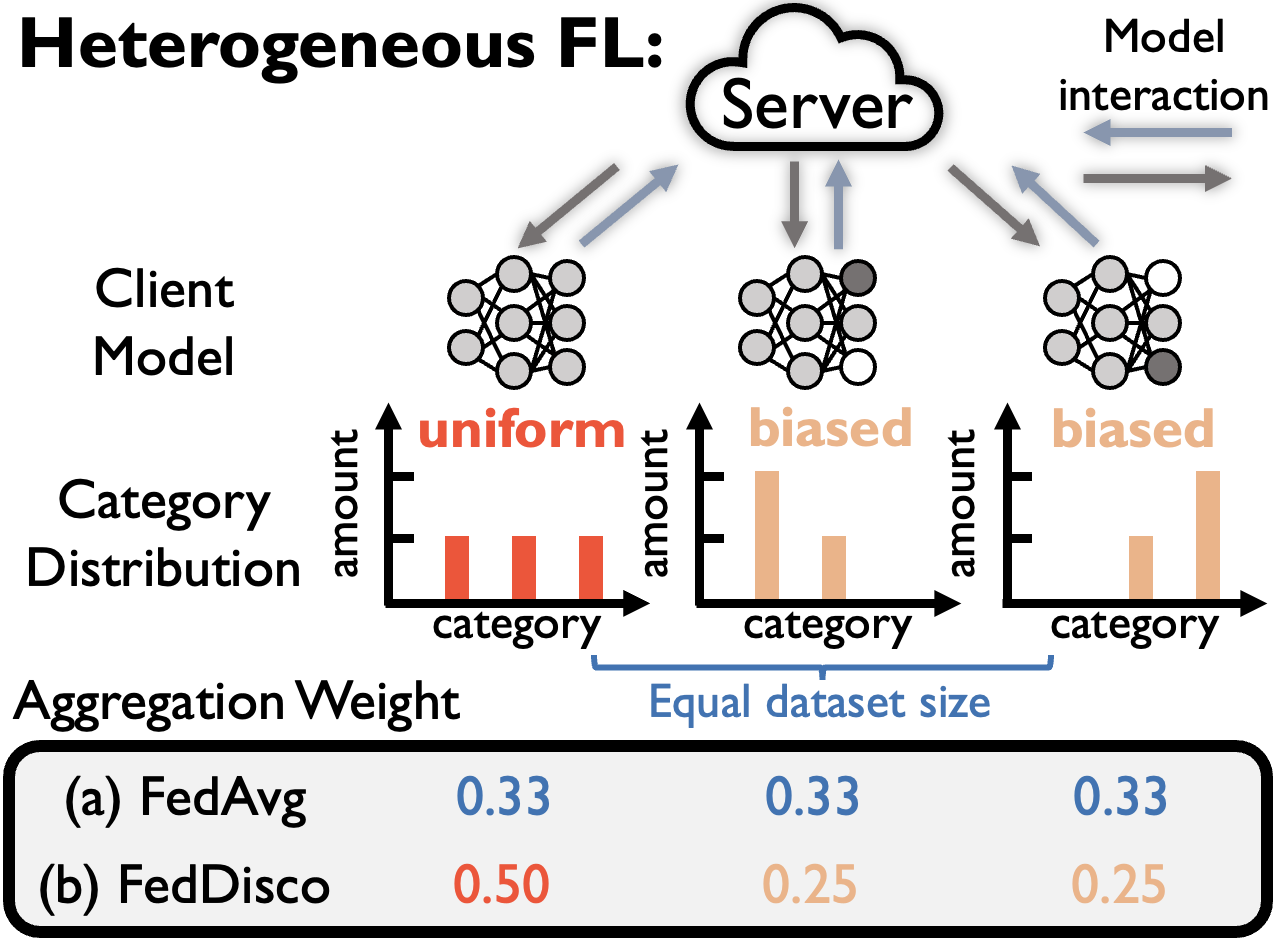}
\caption{To mitigate the negative impact of category distribution heterogeneity, FedDisco considers discrepancy-aware  aggregation weights, which involves both dataset size and discrepancy between local category distribution and uniform distribution. FedDisco is still privacy preserving, communication and computation efficient.}
\label{figure:intro}
\vskip -0.2in
\end{figure}

To mitigate the negative impact of category distribution heterogeneity, there are mainly two approaches: i) local model adjustment at the client side and ii) global model adjustment at the server side. Most previous works consider the first approach. For example, FedProx~\cite{fedprox} and FedDyn~\cite{feddyn} introduce regularization terms; and MOON~\cite{moon} aligns intermediate outputs of local and global models to reduce the variance among optimized local models. Along the line of the second approach, FedDF and FedFTG~\cite{feddf, fedftg} introduce an additional fine-tuning step to refine the global model. Despite all these diverse efforts, few works consider optimizing the aggregation weight at the server side\footnote{FedNova~\cite{fednova} targets system heterogeneity.}. In fact,  most previous works simply aggregate the local models according to the dataset size at each local client. However, the dataset size does not reflect any categorical information and thus cannot provide sufficient information of a local client. It is still unclear whether the dataset size based aggregation is the best aggregation strategy.

In this paper, we first answer the above open question by empirically showing that, in many cases, aggregating local models based on the local dataset size is consistently far from optimal; meanwhile, incorporating the discrepancy between local and global category distributions into the aggregation weights could be beneficial. Intuitively, a lower discrepancy value reflects that the private data at a local client has a more similar category distribution with the hypothetically aggregated global data, and the corresponding client should contribute more to the global model than those with higher discrepancy values. To understand the effect of category distribution heterogeneity, we next provide a theoretical analysis for federated averaging with arbitrary aggregation weights, dataset sizes and local discrepancy levels. By minimizing the reformulated upper bound of the optimization error, we obtain an optimized aggregation weights for clients, which depend on not only the local dataset sizes but also the local discrepancy levels.

Inspired by the above empirical and theoretical observations, we propose a novel model aggregation method, federated learning with discrepancy-aware collaboration (FedDisco), to alleviate the category distribution heterogeneity issue. This new method leverages each client's dataset size and discrepancy in determining aggregation weight by assigning larger aggregation weights to those clients with larger dataset sizes and smaller discrepancy values. As a novel model aggregation scheme, FedDisco is still privacy-preserving and can be easily incorporated with many FL methods, such as FedProx~\cite{fedprox} and MOON~\cite{moon}. Compared to the vanilla aggregation based on local dataset size, FedDisco nearly introduces no additional computation and communication costs.

To validate the effectiveness of our proposed FedDisco, we conduct extensive experiments under diverse scenarios, such as various heterogeneous settings, globally balanced and imbalanced distributions, full and partial client participation, and across four datasets. We observe that FedDisco can significantly and consistently improve the performance over existing FL algorithms.

The main contributions of this paper is listed as follows:
\begin{enumerate}
  \item We empirically show that the dataset size is not optimal to determine aggregation weights in FL and the discrepancy between local and global category distributions could be a beneficial complementary indicator;
  \item We theoretically show that aggregation weight correlated with both dataset size and discrepancy could contribute to a tighter error bound;
    \item We propose a novel aggregation method, called FedDisco (federated learning with discrepancy-aware collaboration), whose aggregation weight follows from the optimization of the theoretical error bound;
    \item Extensive experiments show that FedDisco achieves state-of-the-art performances under diverse scenarios.
\end{enumerate}

% ----------------------------
\section{Background of Federated Learning}
\label{sec:background}
Suppose there are total $K$ clients, where the $k$-th client holds a private dataset $\mathcal{B}_k=\{ (\mathbf{x}_i,y_i) | i=1,2,...,|\mathcal{B}_k|\}$ and a local model $\mathbf{w}^{(t, r)}_k$, where $\mathbf{x}_i$ and $y_i$ denote the $i$-th input and label, $t$ and $r$ denote the round and iteration indices of training. The local category distribution of $k$-th client's dataset is defined as $\mathbf{D}_k \in \mathbb{R}^C$, where the $c$-th element $\mathbf{D}_{k,c}=\frac{|\{ (\mathbf{x}_i,y_i) | y_i=c \}|}{|\mathcal{B}_k|}, c \in \{1,2,\cdots,C\}$ is the data proportion of the $c$-th catregory. See detailed notation descriptions in~\cref{table:notation}.

Mathematically, the global objective of federated learning is $\mathop{\min}_{\mathbf{w}} F (\mathbf{w})=\sum_{k=1}^K n_k F_k(\mathbf{w})$,
where $n_k=\frac{| \mathcal{B}_k|}{\sum_{i=1}^{K} \mathcal{B}_i}$ and $F_k(\mathbf{w})$ are the dataset relative size and local objective of client $k$, respectively. In the basic FL, FedAvg~\cite{fedavg}, each training round $t$ proceeds as follows:
\begin{enumerate}
    \item Sever broadcasts the global model $\mathbf{w}^{(t,0)}$ to clients;
    \item Each client $k$ performs local model training using $\tau$ SGD steps to obtain a trained model denoted by $\mathbf{w}_k^{(t,\tau)}$;
    \item Clients upload the local models to the server;
    \item Server updates the global model based on the aggregated local models: $\mathbf{w}^{(t+1,0)}=\sum_{k=1}^K p_k \mathbf{w}_k^{(t,\tau)}$, where $p_k$ is the aggregation weight for the client $k$.
\end{enumerate}
As mentioned in introduction, the category distribution heterogeneity issue could cause different local objectives in local training. To address this, many previous works focus on adjustment on training $\mathbf{w}_k^{(t,\tau)}$, while neglects the adjustment of aggregation weight $p_k$. In fact, most previous works conduct model aggregation simply based on the local dataset relative size; that is, $p_k=n_k$. However, we notice that dataset-size-based weighted aggregation could be not optimal in empirical investigation, which motivate us to search for a better aggregation strategy. The empirical observations are delivered in the next section.

% -----------------------------------
\section{Empirical Observations}

This section presents a series of empirical observations, which motivate us to consider a new form of aggregation weight and propose our method. In the experiments, we divide the CIFAR-10~\cite{cifar10} dataset into $10$ clients, whose local category distributions are heterogeneous and follow commonly used Dirichlet distribution~\cite{fedma}; see Figure~\ref{figure:motivation} (a). Based on this common setting, we conduct the following experiments.

\begin{figure}[t]  
\vskip 0.2in
	\centering
	\subfigure[Data distributions]{
		\includegraphics[width=0.47\columnwidth]{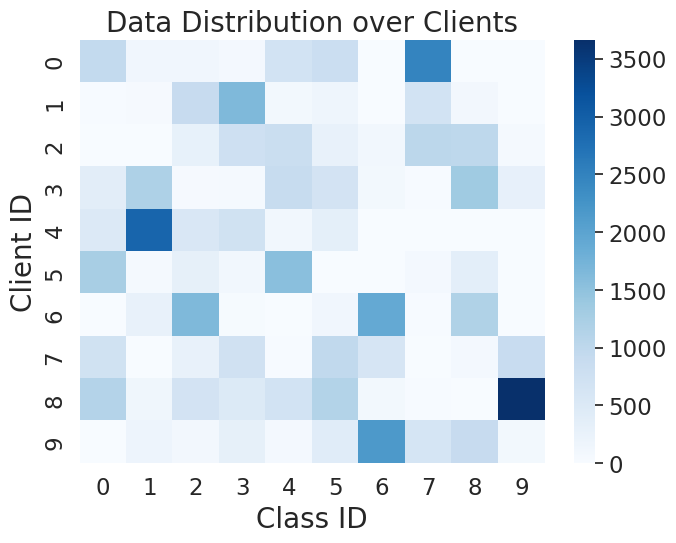}
	}
	\subfigure[Exploration 1: aggregation]{
		\includegraphics[width=0.47\columnwidth]{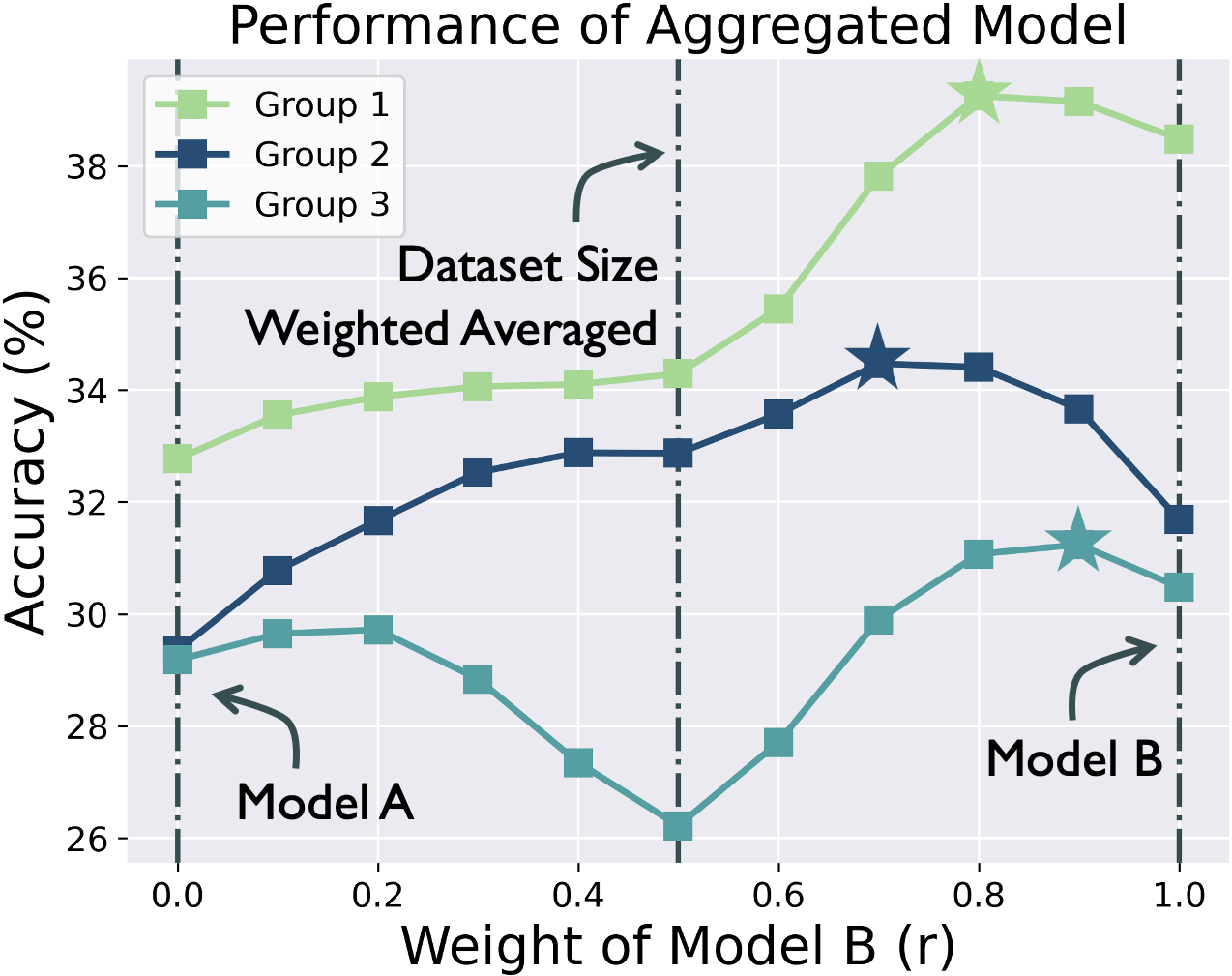}
	}
	\subfigure[Exploration 2: dataset size]{
		\includegraphics[width=0.47\columnwidth]{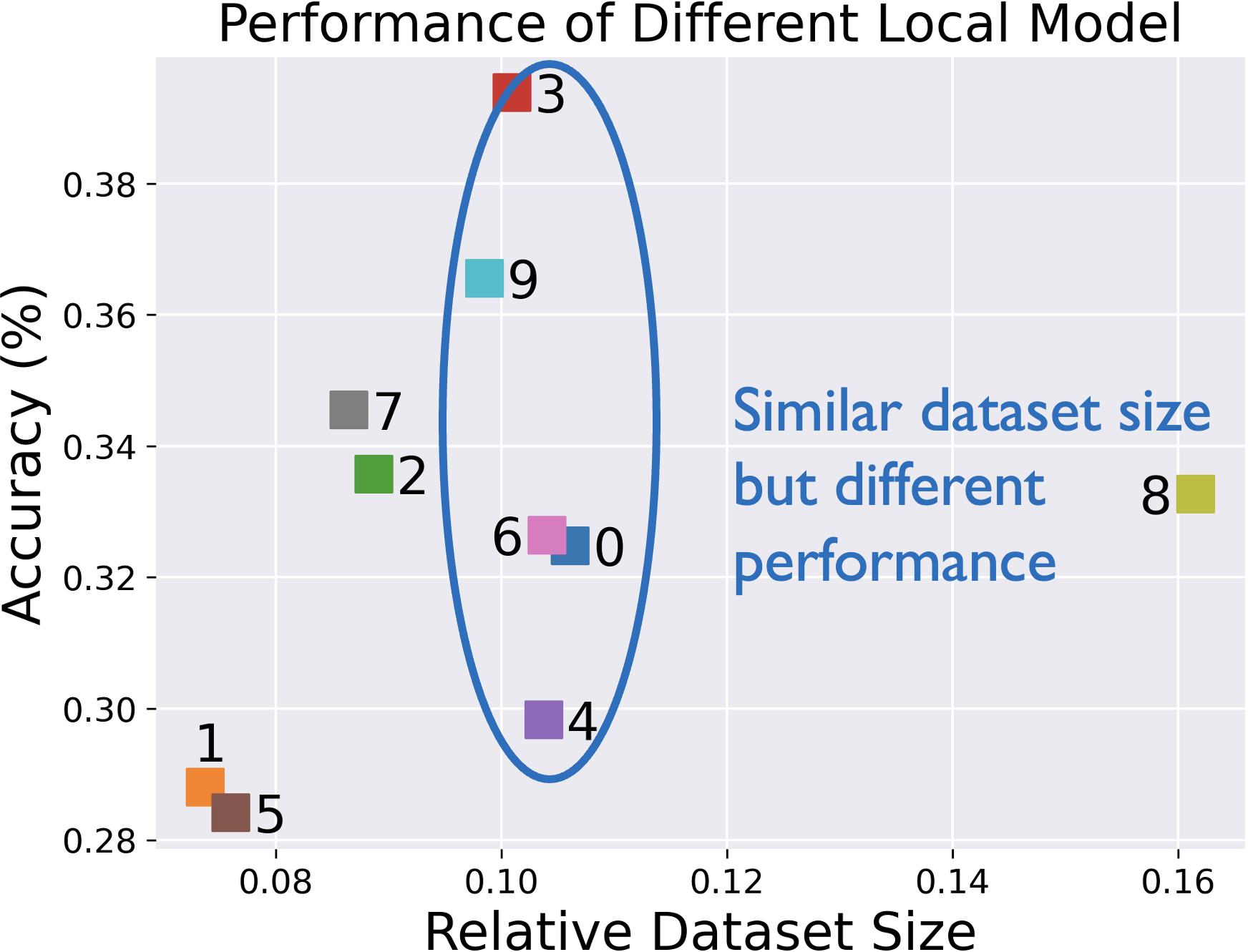}
	}
	\subfigure[Exploration 2: discrepancy]{
		\includegraphics[width=0.47\columnwidth]{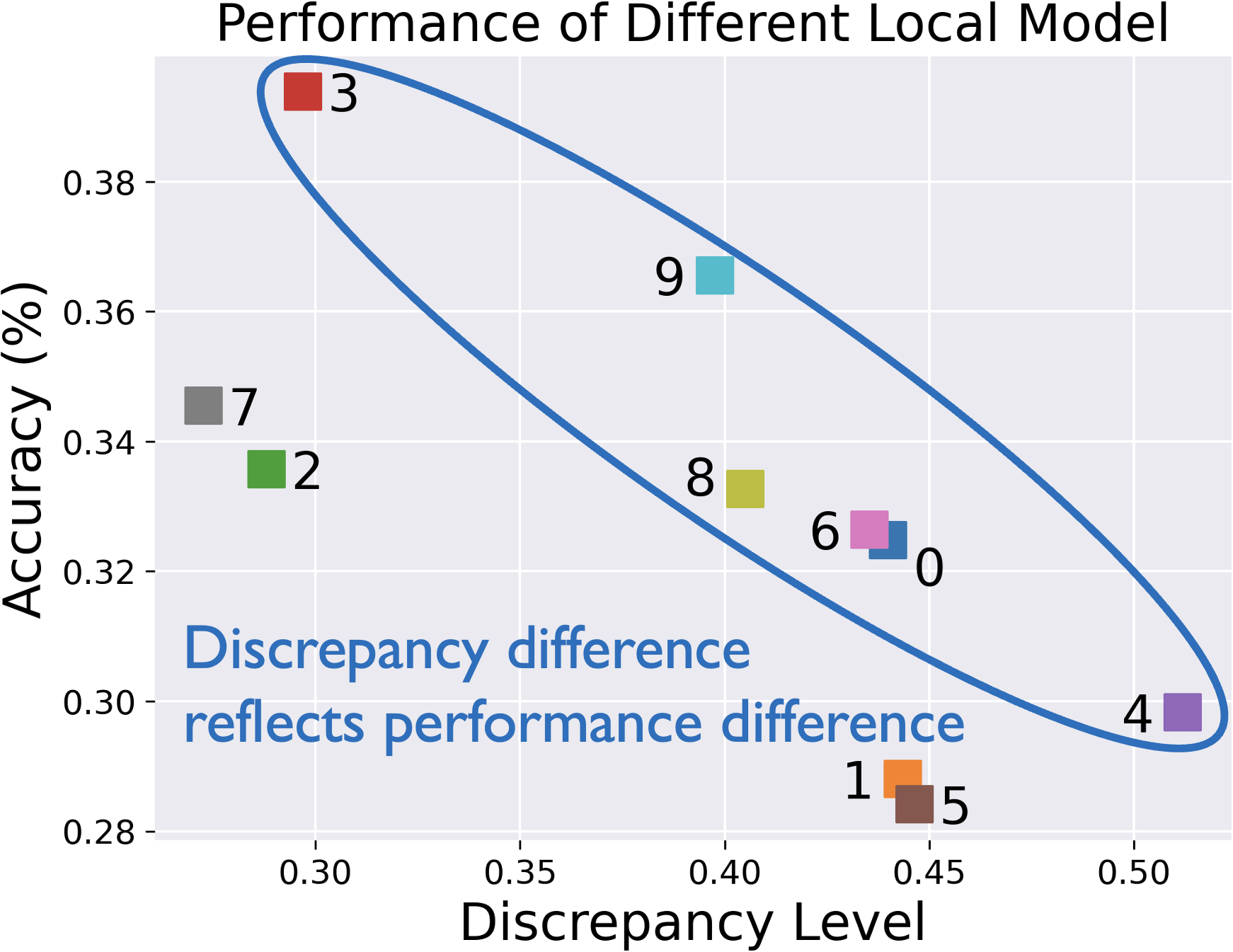}
	}
	\caption{Experiments for empirical observations. Exploration 1 shows that dataset size could be not optimal indicator for aggregation weight. Exploration 2 shows that discrepancy could be a beneficial complementary indicator.}
	\label{figure:motivation}
\vskip -0.2in
\end{figure}

To explore the relationship between the aggregation weight and the performance of aggregated global model, we conduct three independent trials, where each trial considers federated training of two clients with similar dataset sizes. In each trial, each client trains a model (denoted as A and B) for $10$ epochs and we aggregate these two models following: $\mathbf{w}_{\rm agg} = (1-r)\mathbf{w}_{\rm A}+r\mathbf{w}_{\rm B}$. Figure~\ref{figure:motivation} (b) shows three curves, which are the testing results of the aggregated model $\mathbf{w}_{\rm agg}$ as a function of aggregation weight $r$ in three trials, respectively. Note that $r=0$ reflects Model A,  $r=1$ reflects Model B, and $r=0.5$ represents dataset-size-based weighted aggregation. We observe that i) {\bf it is not optimal to determine aggregation weights purely based on local dataset size.}  The performances at $r=0.5$ could be far from the optimal performances (stars in the plot); and ii) {\bf best performance is achieved when assigning a relatively larger weight $r$ to a better-performed local model.} In these trials, Model B outperforms Model A, and the best performance is achieved when Model B has a larger weight ($r$ is around $0.7\sim0.9$). Similar phenomena can be seen in ensemble learning~\cite{jimenez1998dynamically,shen2004dynamically}, which assigns larger weight to better model in ensemble.

To search for indicators that can reflect local model's performance and eventually appropriate aggregation weight, we explore the effects of two informative properties of client in category heterogeneity scenario: dataset size and discrepancy between local and global category distribution. Here we use $\ell_2$ distance to measure the discrepancy. We plot all $10$ clients in Figure~\ref{figure:motivation} (c) \& (d), where y-axis denote the local model's testing accuracy and x-axis in (c) and (d) denotes client's dataset relative size and discrepancy level, respectively. We highlight four clients in a circle, which have similar dataset sizes. Comparing these two plots, we clearly see that {\bf the discrepancy level is a better indicator to reflect the local model's performance than the dataset size.} In plot (c), we see that these clients have similar dataset sizes but largely different performances; while in plot (d), the discrepancy level clearly reflects the performance difference among these four clients, that is, the client with a smaller discrepancy performs better.

Based on these observations, we hypothesize that, in FL, when a client has a smaller discrepancy value, it might have a better-performed model and thus need a larger weight in aggregation. Motivated by this, we propose to leverage discrepancy in determining the aggregation weight of each client, which is theoretically analyzed in the following.

\section{Theoretical Analysis}
\label{sec:theory}

In this section, we firstly obtain a convergence error bound for FedAvg~\cite{fedavg}, which highlights the effect of aggregation weight $p_k$, dataset size $n_k$ and discrepancy level $d_k$. Then, a concise analytical expression of the optimized aggregation weight is derived by minimizing the error bound, which indicates that aggregation weight should depend on both dataset size and local discrepancy level.

\paragraph{Optimization error upper bound.}
Our analysis is based on the following four standard assumptions in FL.

\begin{assumption}[Smoothness]
\label{ass_smooth}
Function $F_k(\mathbf{w})$ is Lipschitz-smooth: $||\nabla F_k(\mathbf{x})- \nabla F_k(\mathbf{y})|| \leq L ||\mathbf{x}-\mathbf{y}||$ for some $L$.
\end{assumption}

\begin{assumption}[Bounded Scalar]
\label{ass_scalar}
The global objective function $F(\mathbf{w})$ is bounded below by $F_{inf}$.
\end{assumption}

\begin{assumption}[Unbiased Gradient and Bounded Variance]
\label{ass_grad}
For each client, the stochastic gradient is unbiased: $\mathbb{E}_\xi[g_k(\mathbf{w}|\xi)] = \nabla F_k(\mathbf{w})$, and has bounded variance: $\mathbb{E}_\xi[||g_k(\mathbf{w}|\xi)-\nabla F_k(\mathbf{w})||^2] \leq \sigma^2$.
\end{assumption}

\begin{assumption}[Bounded Dissimilarity]
\label{ass_diss}
For each loss function $F_k(\mathbf{w})$, there exists constant $B > 0$ such that $||\nabla F_k(\mathbf{w})||^2 \leq ||\nabla F(\mathbf{w})||^2 + B d_k$.
\end{assumption}

All assumptions are commonly used in federated learning literature \cite{fednova,fedprox,li2019convergence,reddi2021adaptive}. As no previous literature has considered discrepancy in their theory, a new assumption is required to include discrepancy. However, rather than introducing an additional assumption, we slightly adjust standard dissimilarity assumption~\cite{fednova} and apply~\cref{ass_diss} to include the discrepancy level, this modification correlates the gradient dissimilarity and distribution discrepancy, and enables us to explore the relationships among aggregation weight $p_k$, dataset relative size $n_k$ and discrepancy $d_k$.

We present the optimization error bound in~\cref{theorem_1}.
\begin{theorem}[Optimization bound of the global objective function]
\label{theorem_1}
Let $F (\mathbf{w})=\sum_{k=1}^K n_k F_k(\mathbf{w})$ be the global objective. Under these Assumptions, if we set $\eta L \leq \frac{1}{2 \tau}$, the optimization error will be bounded as follows:
\begin{equation*}
\small
\begin{aligned}
% \label{eq:optimization bound}
& \mathop{\min}_{t} \mathbb{E} ||\nabla F (\mathbf{w}^{(t,0)})||^2 \leq \underbrace{\frac{1}{1-3A-W_D(1-A)}}_{T_0} \\
& \bigg( \underbrace{\frac{2(1-A)[F(\mathbf{w}^{(0,0)})-F_{inf}]}{\tau \eta T}}_{T_1} + \underbrace{\frac{(1-A)W_DB}{K}\sum_{k=1}^Kd_k}_{T_2} \\
& + \underbrace{2(1-A)L\eta \sigma^2 \sum_{k=1}^K p_k^2}_{T_3} + \underbrace{2(\tau-1)\sigma^2 L^2 \eta^2}_{T_4} + \underbrace{2AB\sum_{k=1}^K p_k d_k}_{T_5}  \bigg)
\end{aligned}
\end{equation*}
where $W_D=2\sum_{k=1}^K (n_k-p_k)^2$, $A=2\tau(\tau-1)\eta^2L^2$, $\eta$ is local learning rate, $B$ is a constant in~\cref{ass_diss}.
\end{theorem}
\begin{proof}
    See details in~\cref{append_proof_theorem}.
\end{proof}

Generally, a tighter bound corresponds to a better optimization result. Thus, we explore the effects of $p_k$ on upper bound. In~\cref{theorem_1}, there are four parts related to $p_k$. First, we see that larger difference between $p_k$ and $n_k$ contributes to larger $W_D$ and thus smaller denominator in $T_0$ and larger value in $T_2$, which tends to loose the bound. As for $T_5$, by setting $p_k$ negatively correlated to $d_k$, when clients have different level of discrepancy, i.e. different $d_k$, $T_5$ will be further reduced, which tends to tight the bound.

Therefore, there could be an optimal set of $\{p_k | k \in [K]\}$ that contributes to the tightest bound, where an optimal $p_k$ should be correlated to both $n_k$ and $d_k$. This theoretically show that dataset size could be not the optimal indicator and that discrepancy level can be a reasonable complementary indicator for determining aggregation weight.

\paragraph{Upper bound minimization.}

In the following, we derive the analytical expression of aggregation weight ($p_k$) by minimizing the upper bound in~\cref{theorem_1}. To minimize this upper bound, directly solving the minimization results in a complicated expression, which involves too many unknown hyper-parameters in practice. To simplify the expression, we convert the original objective from minimizing $(T_1+T_2+T_3+T_4+T_5)/T_0$ to minimizing $T_1+T_2+T_3+T_4+T_5 - \lambda T_0$, where $\lambda$ is a hyper-parameter. The converted objective still promotes maximization of $T_0$ and minimization of $T_1+T_2+T_3+T_4+T_5$, and still contributes to tighten the bound $(T_1+T_2+T_3+T_4+T_5)/T_0$ (also see analysis in~\cref{sec:numerical_simulation}). Then, our discrepancy-aware aggregation weight is obtained through solving the following optimization problem:
\begin{align}
\label{eq:reformulated}
\mathop{\min}_{\{p_k\}} & \frac{2(1-A)[F(\mathbf{w}^{(0,0)})-F_{inf}]}{\tau \eta T} + \frac{(1-A)W_DB}{K}\sum_{k=1}^Kd_k \notag\\
& + 2(1-A)L\eta \sigma^2 \sum_{k=1}^K p_k^2 + 2(\tau-1)\sigma^2 L^2 \eta^2 \notag\\ 
& + 2AB\sum_{k=1}^K p_k d_k -
\lambda \left( 1-3A-W_D(1-A) \right), \notag\\
{\rm s.t.} & \sum_{k} p_k =1,~p_k \geq 0,
\end{align}
from which we derive the concise expression of an optimized aggregation weight (see details of derivation in~\cref{append_proof_obtain_pk}):
\begin{equation}
\label{eq:disco_weight}
    p_k \propto n_k - a \cdot d_k + b,
\end{equation}
where $a, b$ are two constants. This suggests that for a tighter upper bound, the aggregation weight $p_k$ should be correlated with both dataset size $n_k$ and local discrepancy level $d_k$. This expression of Disco aggregation weight can mitigate the limitation of standard dataset size weighted aggregation by assigning larger aggregation weight to client with larger dataset size and smaller discrepancy level.

\begin{figure*}[t]
% \vskip 0.2in
\centering
\includegraphics[width=0.9\textwidth]{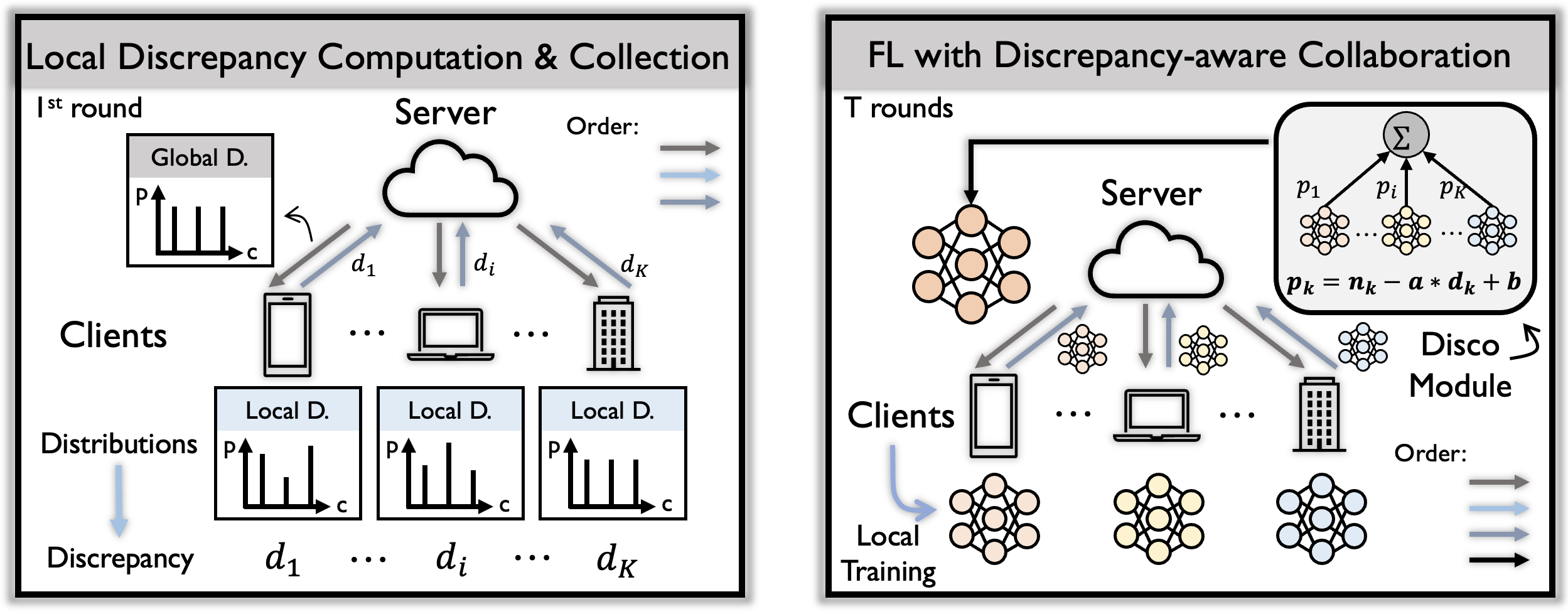}
\caption{The overview of federated learning with discrepancy-aware collaboration. The left shows the acquisition of clients' discrepancy by comparing global and local category distribution, which is only required at the first round. The right shows the federated learning process with discrepancy-aware collaboration by adjusting each client's aggregation weight based on dataset size and discrepancy level.}
\label{figure:overview}
\vskip -0.2in
\end{figure*}

% ----------------------------------------------
\section{FL with Discrepancy-Aware Collaboration}
\label{sec:method}

Motivated by these empirical and theoretical observations, we propose federated learning with discrepancy-aware collaboration (FedDisco).

\subsection{Framework} 
FedDisco involves four key steps: local~model training, local discrepancy computation, discrepancy-aware aggregation weight determination and global model aggregation. Steps 1 and 4 follows from the standard federated learning; meanwhile, Steps 2 and 3 integrates the discrepancy between local and global category distributions into aggregation weights. The overview is shown on the right of Figure~\ref{figure:overview}. We also provide an algorithm flow in~\cref{alg:feddisco}.

\textbf{Local~model~training.} Each client $k$ performs $\tau$ steps of local training based on dataset $\mathcal{B}_k$ and initial model $\mathbf{w}^{(t,0)}$ to obtain the trained model, denoted as 
$ \mathbf{w}_k^{(t,\tau)} = LocalTrain(\mathcal{B}_k,\mathbf{w}^{(t,0)}).$

\textbf{Local discrepancy computation.} Each client needs to compute the discrepancy between its local category distribution and the hypothetically aggregated global category distribution. Here we consider that the  global category distribution is uniform, since it naturally promotes the fairness across all categories and the generalization capability of the global model. Moreover, each local client can compute its discrepancy without additional data sharing, preventing information leakage of category distribution. Concretely, $\mathbf{D}_k$ and $\mathbf{T}$ denote local and global category distribution. Thus, we set all elements $\mathbf{T}_c = {1}/{C}$ to treat all categories equally. Given the global distribution $\mathbf{T}$, each client $k$ can calculate its local discrepancy level $d_k \in \mathbb{R}$ by evaluating the difference between these two distributions: 
$
d_k = e(\mathbf{D}_k,\mathbf{T}),
$
where $e(\cdot)$ is a pre-defined metric function (e.g. L2 difference or KL-Divergence). Applying L2 difference corresponds to $d_k=\sqrt{\sum_{c=1}^C (\mathbf{D}_{k,c}-\mathbf{T}_c)^2}$. Finally, the server collects clients' discrepancy levels, see the left of Figure~\ref{figure:overview}. Note that though we are more interested in uniform target distribution for the sake of category-level fairness, our method is also applicable to scenarios where target distribution is imbalanced; see experiments in~\cref{table:imb_test}.

\textbf{Discrepancy-aware~aggregation~weight~determination.} Motivated by the previous empirical and theoretical observations, we propose to determine more distinguishing aggregation weights for each client $k$ by leveraging dataset relative size $n_k$ and local discrepancy level $d_k$ (derived from~\eqref{eq:disco_weight}):
\begin{equation}
\label{eq:disco_real}
    p_k = \frac{{\rm ReLU}(n_k - a \cdot d_k + b) }{\sum_{m=1}^K{\rm ReLU}(n_m - a \cdot d_m + b)},
\end{equation}
where ReLU$(\cdot)$ is the relu function to take care of negative values~\cite{fukushima1975cognitron}, $a$ is a hyper-parameter to balance $n_k$ and $d_k$, $b$ is another hyper-parameter to adjust the weight. This Disco aggregation can determine more distinguishing aggregation weights for clients by assigning larger weight for clients with larger dataset sizes and smaller local discrepancy levels.

\textbf{Global model aggregation.} The server conducts model aggregation based on the above discrepancy-aware aggregation weights. The global model is $\mathbf{w}^{(t+1)}=\sum_{k=1}^K p_k \mathbf{w}_k^{(t)},$
where $\mathbf{w}_k^{(t)}$ is the $k$th local model.

\subsection{Discussions}

\textbf{Privacy.} 
Our proposed FedDisco does not leak client's exact distribution (Figure~\ref{figure:overview}) since the discrepancy is calculated at the client side. The server collects discrepancy level of clients, from which the exact category distribution can not be inversely inferred. This indicates that this process is more privacy-preserving compared with several existing works, such as FedFTG and CCVR~\cite{fedftg,nofear}, which transmitting exact category distribution.

\textbf{Communication \& computation efficiency.} Discrepancy communication is only required at the \emph{first communication round}. Besides, the discrepancy is only a numerical value, which is negligible compared with the model communication.  The discrepancy calculation is only a simple operation between two vectors, which is negligible compared with model training. Moreover, FedDisco requires 
much less computation overhead and no computation burden to the server; while previous methods, such as FedDF~\cite{feddf} and FedFTG~\cite{fedftg}, require additional model fine-tuning by simultaneously running $K$ local models at the server side for \emph{every communication round}.

\textbf{Modularity.} 
The modularity of FedDisco indicates its broad range of application. Typically, federated learning involves four steps: (1) global model downloading, (2) local updating, (3) local model uploading and (4) model aggregation. Our FedDisco focuses on step 4, which suggests that it can be easily combined with existing works conducting correction in step 2 and compression in step 1 and 3. Beyond this, it could also be incorporated with methods in step 4 by adjusting aggregation weight to be negatively correlated with discrepancy, such as conducting Disco re-weighting before normalization in FedNova~\cite{fednova}.

\begin{table*}[t]
\caption{Accuracy comparisons (mean $\pm$ std on $5$ trials, \%) on several heterogeneous settings and datasets. Experiments show that FedDisco consistently outperforms these state-of-the-art methods.}
\label{table:main}
\setlength\tabcolsep{5pt}
\vskip 0.15in
\begin{center}
\begin{small}
\begin{sc}
\begin{tabular}{cccccccc}
\toprule
\multirow{2}{*}{Method} & HAM10000 & \multicolumn{2}{c}{CIFAR-10} & \multicolumn{2}{c}{CINIC-10} & \multicolumn{2}{c}{Fashion-MNIST} \\
& NIID-1 & NIID-1  & NIID-2 & NIID-1 & NIID-2 & NIID-1& NIID-2 \\
\midrule
FedAvg  & 42.54$\pm 0.59$ & 68.47$\pm 0.20$ & 65.60$\pm 0.16$ & 54.24$\pm 0.18$ & 50.35$\pm 0.43$ & 89.26$\pm 0.09$ & 86.46$\pm 0.03$ \\
FedAvgM & 42.54$\pm 0.45$ & 68.59$\pm 0.32$ & 66.01$\pm 0.25$ & 54.00$\pm 0.44$ & 50.23$\pm 0.64$ & 89.31$\pm 0.08$ & 87.06$\pm 0.11$ \\
FedProx & 44.76$\pm 1.35$ & 69.33$\pm 0.49$ & 65.61$\pm 0.31$ & 56.38$\pm 0.35$ & 50.30$\pm 0.44$ & 89.28$\pm 0.12$ & 87.24$\pm 0.21$ \\
SCAFFOLD& 55.08$\pm 0.23$ & 71.09$\pm 0.24$ & 66.93$\pm 0.17$ & 57.47$\pm 0.35$ & 53.50$\pm 0.43$ & 89.65$\pm 0.07$ & 86.87$\pm 0.24$ \\
FedDyn  & 54.44$\pm 0.98$ & 70.14$\pm 0.31$ & 68.49$\pm 0.62$ & 56.41$\pm 0.41$ & 52.69$\pm 0.52$ & 89.12$\pm 0.13$ & 86.38$\pm 0.43$ \\
FedNova & 44.92$\pm 0.59$ & 68.57$\pm 0.19$ & 65.61$\pm 0.39$ & 54.27$\pm 0.20$ & 50.37$\pm 0.73$ & 89.05$\pm 0.08$ & 86.47$\pm 0.20$ \\
MOON    & 45.87$\pm 0.90$ & 67.42$\pm 0.14$ & 66.25$\pm 0.35$ & 52.08$\pm 0.20$ & 50.42$\pm 0.56$ & 89.29$\pm 0.04$ & 86.44$\pm 0.04$ \\
FedDC   & 54.48$\pm 0.77$ & 70.93$\pm 0.15$ & 67.89$\pm 0.44$ & 57.26$\pm 0.26$ & 52.43$\pm 0.60$ & 89.19$\pm 0.02$ & 86.57$\pm 0.06$ \\
\textbf{FedDisco}  & \textbf{59.05}$\pm 0.67$ & \textbf{72.05}$\pm 0.24$ & \textbf{69.85}$\pm 0.18$ & \textbf{58.07}$\pm 0.15$ & \textbf{53.84}$\pm 0.08$ & \textbf{89.74}$\pm 0.07$ & \textbf{87.85}$\pm 0.20$ \\
\bottomrule
\end{tabular}
\end{sc}
\end{small}
\end{center}
\vskip -0.2in
\end{table*}

\begin{table*}[t]
\caption{Modularity. Each entry shows accuracy of baseline with Disco (accuracy \textbf{difference} compared with baseline without Disco in Table~\ref{table:main}). Experiments show consistent improvements across settings, indicating the modularity of FedDisco.}
\label{table:modularity}
\setlength\tabcolsep{3pt}
\vskip 0.15in
\begin{center}
\begin{small}
\begin{sc}
\begin{tabular}{cccccccc}
\toprule
\multirow{2}{*}{\textbf{+ Disco}} & HAM10000 & \multicolumn{2}{c}{CIFAR-10} & \multicolumn{2}{c}{CINIC-10} & \multicolumn{2}{c}{Fashion-MNIST}\\
    & NIID-1 & NIID-1  & NIID-2 & NIID-1 & NIID-2 & NIID-1& NIID-2\\
\midrule
FedAvg   & 50.95 (\textbf{+8.41}) & 70.05 (\textbf{+1.58}) & 68.30 (\textbf{+2.70}) & 54.81 (\textbf{+0.57}) & 52.46 (\textbf{+2.11}) & 89.56 (\textbf{+0.30}) & 87.56 (\textbf{+1.10}) \\
FedavgM   & 50.00 (\textbf{+7.46}) & 70.07 (\textbf{+1.48}) & 67.73 (\textbf{+1.72}) & 54.69 (\textbf{+0.69}) & 51.91 (\textbf{+1.68}) & 89.36 (\textbf{+0.05}) & 87.46 (\textbf{+0.40}) \\
FedProx  & 50.48 (\textbf{+5.72}) & 70.68 (\textbf{+1.35}) & 68.33 (\textbf{+2.72}) & 56.93 (\textbf{+0.55}) & 52.62 (\textbf{+2.32}) & 89.58 (\textbf{+0.27}) & 87.72 (\textbf{+0.48}) \\
SCAFFOLD  & 57.62 (\textbf{+2.54}) & 71.70 (\textbf{+0.61}) & 69.10 (\textbf{+2.17}) & 58.05 (\textbf{+0.58})  & 53.82 (\textbf{+0.32}) & 89.74 (\textbf{+0.09}) & 87.85 (\textbf{+0.98}) \\
FedDyn  & 59.05 (\textbf{+4.61}) & 72.05 (\textbf{+1.91}) & 69.85 (\textbf{+1.36}) & 58.07 (\textbf{+1.66}) & 53.84 (\textbf{+1.15}) & 89.31 (\textbf{+0.19}) & 87.18 (\textbf{+0.80}) \\
FedNova  & 51.43 (\textbf{+6.51}) & 70.04 (\textbf{+1.47}) & 67.83 (\textbf{+2.22}) & 55.04 (\textbf{+0.77}) & 52.23 (\textbf{+1.86}) & 89.28 (\textbf{+0.23}) & 87.52 (\textbf{+1.05}) \\
MOON   & 52.86 (\textbf{+6.99}) & 68.79 (\textbf{+1.37}) & 68.35 (\textbf{+2.10}) & 53.26 (\textbf{+1.18}) & 51.97 (\textbf{+1.55}) & 89.50 (\textbf{+0.21}) & 87.59 (\textbf{+1.15}) \\
FedDC   & 58.25 (\textbf{+3.77}) & 71.96 (\textbf{+1.03}) & 68.94 (\textbf{+1.05}) & 57.70 (\textbf{+0.44}) & 53.25 (\textbf{+0.82}) & 89.51 (\textbf{+0.32})& 87.11 (\textbf{+0.54}) \\
\bottomrule
\end{tabular}
\end{sc}
\end{small}
\end{center}
\vskip -0.2in
\end{table*}

\begin{figure*}[t]  
% \vskip 0.2in
	\centering
	\subfigure[MOON, NIID-1, CIFAR-10]{
		\includegraphics[width=0.23\textwidth]{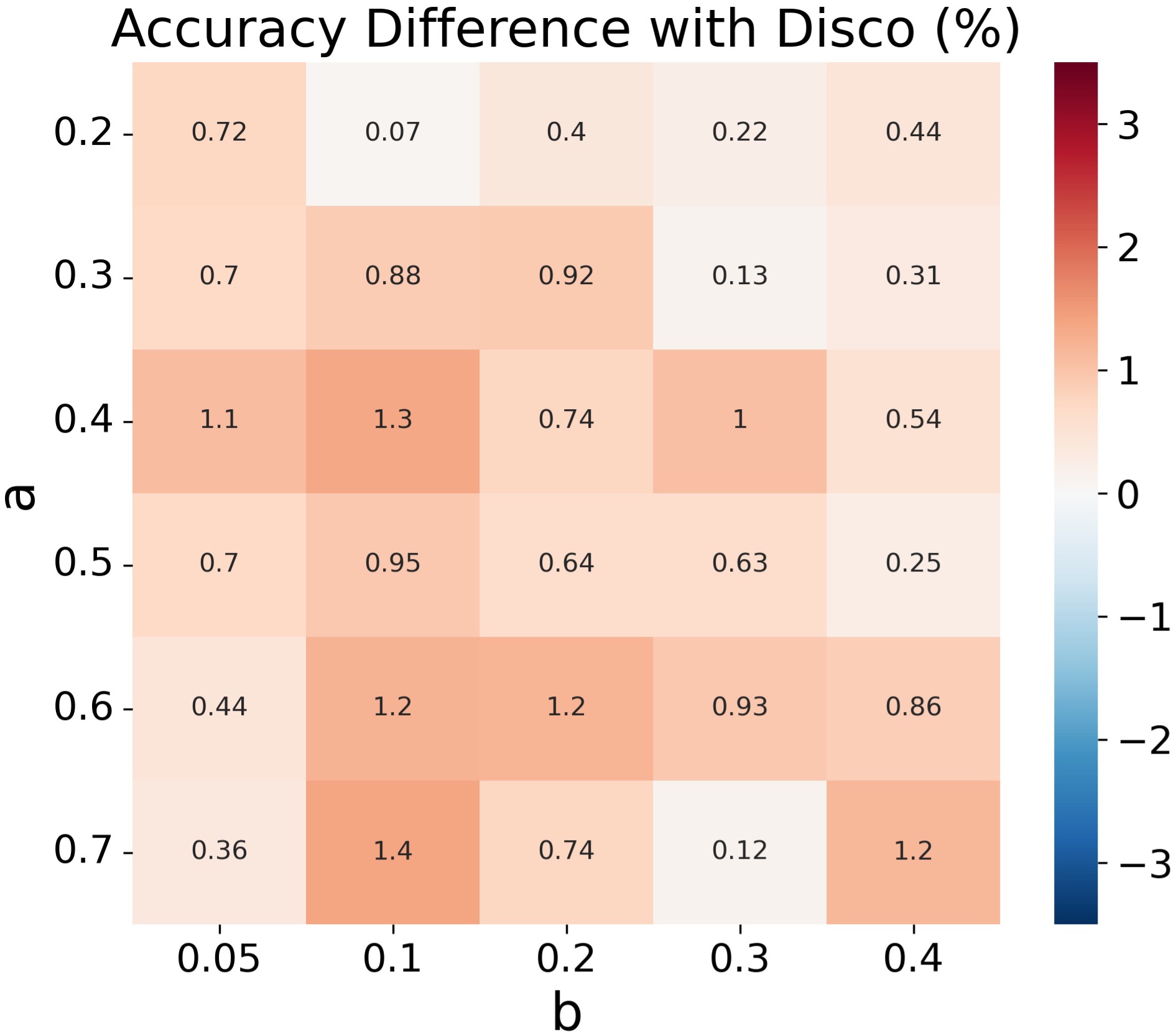}
	}
	\subfigure[FedAvg, NIID-1, CIFAR-10]{
		\includegraphics[width=0.23\textwidth]{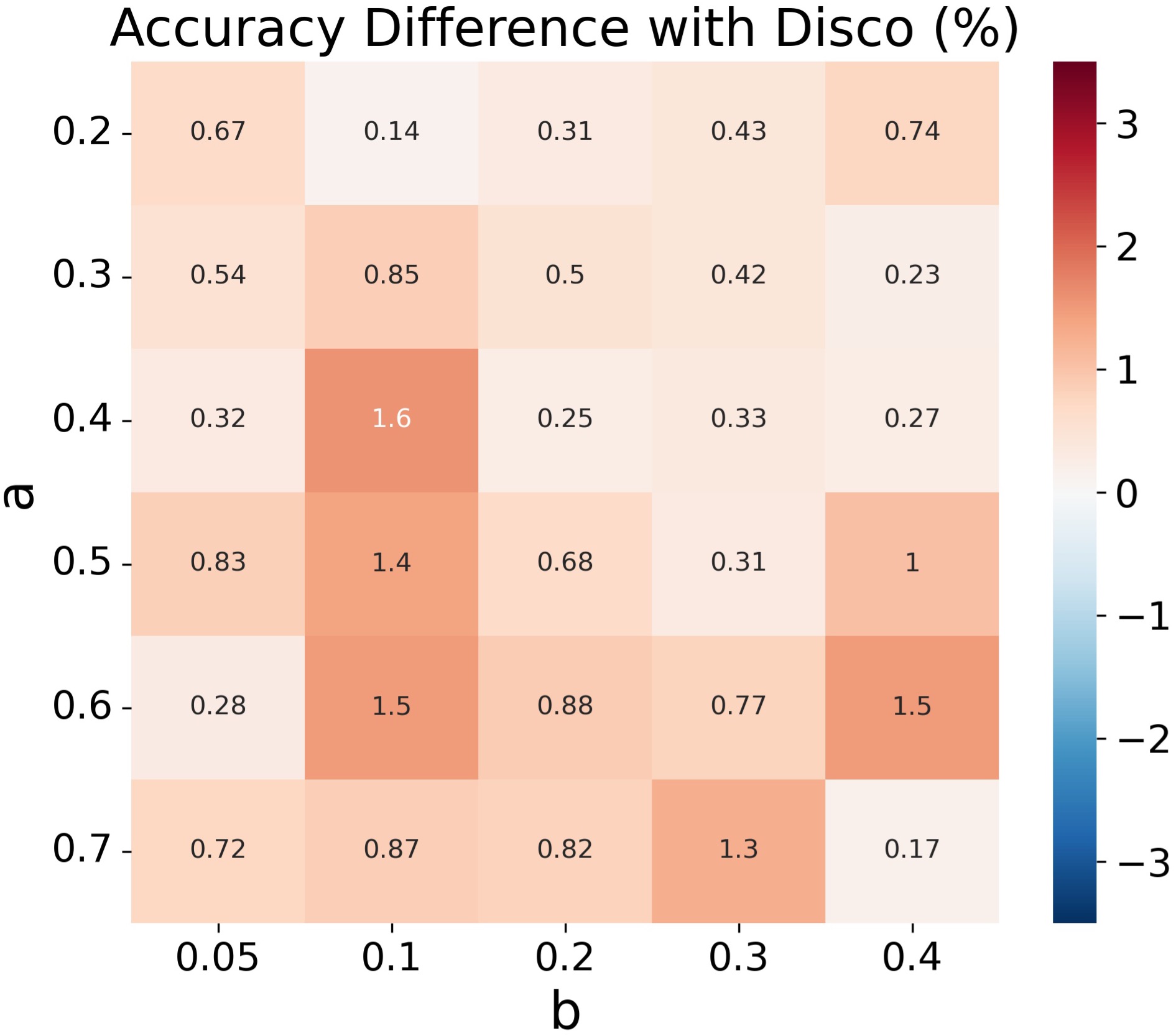}
	}
	\subfigure[FedAvg, NIID-2, CIFAR-10]{
		\includegraphics[width=0.23\textwidth]{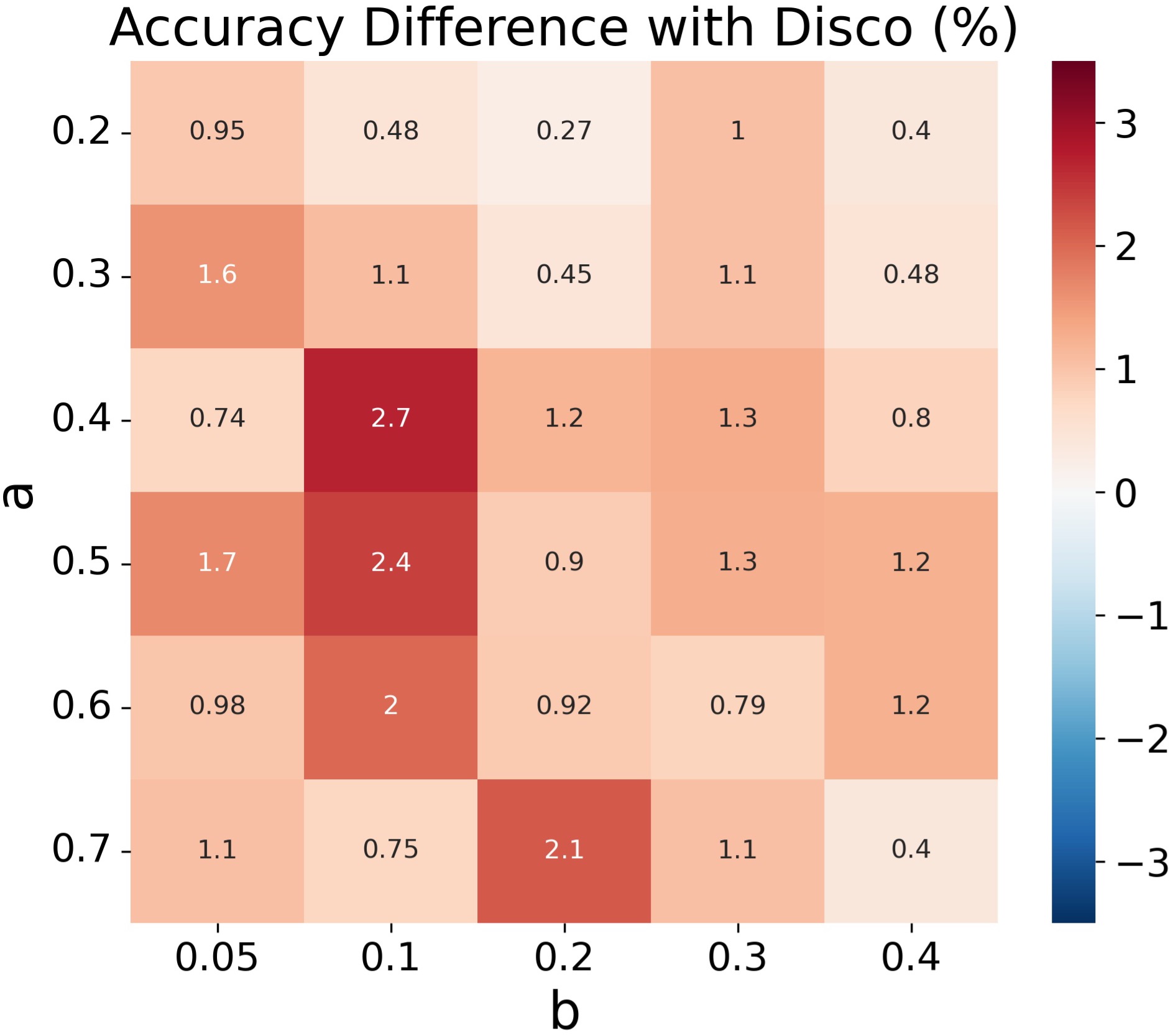}
	}
	\subfigure[FedAvg, NIID-2, CINIC10]{
		\includegraphics[width=0.23\textwidth]{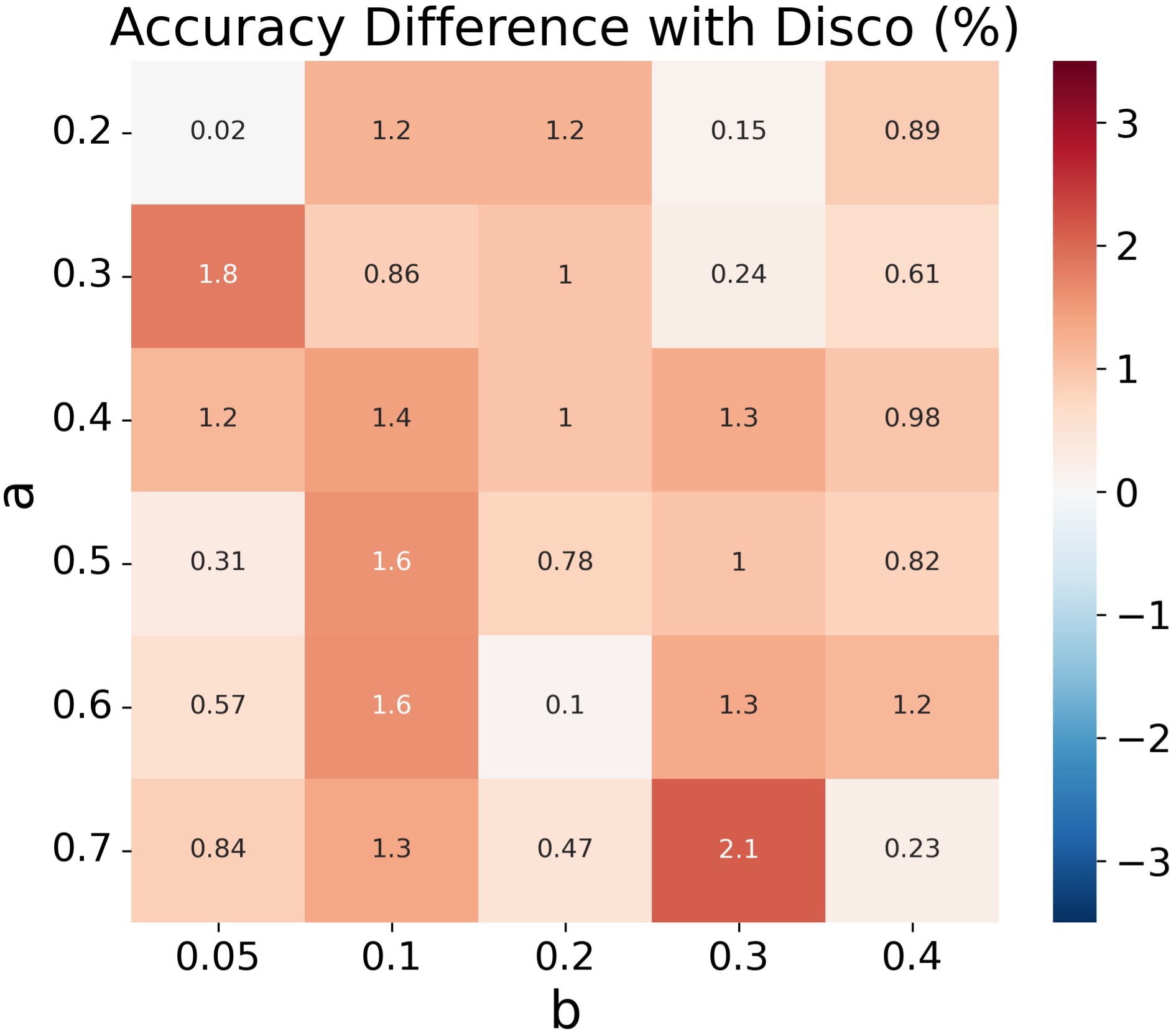}
	}
    \vskip -0.1in
	\caption{Ease of hyper-parameters tuning. Experiments show that our Disco works for a wide range of hyper-parameters.}
	\label{figure:heatmaps}	
\vskip -0.2in
\end{figure*}

% ------------------------------------
\section{Related Works}
\label{sec:related_works}
Federated learning (FL) has been an emerging topic~\cite{fedavg,8859260}. However, data distribution heterogeneity may significantly degrades FL's performance~\cite{zhao2018federated,advances,li2019convergence}. Works that focus on this can be mainly divided into two directions: local and global model adjustment.

\subsection{Local Model Adjustment}
Methods in this direction conduct the adjustment on local model training at the client side, aiming at producing local models with smaller difference~\cite{fedprox,scaffold}. FedProx~\cite{fedprox} regularizes the $\ell_2$-distance between local and global model. FedDyn~\cite{feddyn} proposes a dynamic regularizer to align the local and global solutions. SCAFFOLD~\cite{scaffold} and FedDC~\cite{Gao_2022_CVPR} introduce control variates to correct each client's gradient but require double communication cost. MOON~\cite{moon} aligns the features of local and global model; FedFM~\cite{fedfm} aligns category-wise feature space across clients, contributing to more consistent feature spaces.

All of these methods simply assign aggregation weight based on local dataset size, which is not sufficient to distinguish each client's contribution. While our proposed FedDisco leverages both dataset size and discrepancy between local and global category distributions to determine more distinguishing aggregation weights, which has strong ability of plug-and-play that can be easily incorporated with these methods to further enhance the overall performance.

\subsection{Global Model Adjustment}
Methods in this direction conducts the adjustment on global model at the server side, aiming at produce a global model with better performance~\cite{feddf,reddi2021adaptive,li2023revisiting,zhang2023federated,fan2022fedskip}. FedNova~\cite{fednova} conducts re-weighting to target system heterogeneity while we conduct re-weighting based on local discrepancy level to target data heterogeneity. CCVR~\cite{nofear} conducts post-calibration, FedAvgM~\cite{fedavgm} and FedOPT~\cite{reddi2021adaptive} introduce global model momentum to stabilize FL training; while these methods are orthogonal to FedDisco and can be easily combined. FedDF~\cite{feddf} and FedFTG~\cite{fedftg} introduce an additional fine-tuning step to refine the global model; FedGen~\cite{fedgen} learns a feature generator for assisting local training. However, these methods require great computation capability of the server with much more computation cost (e.g. FedFTG~\cite{fedftg} requires twice computation cost compared with FedAvg~\cite{fedavg}). As a contrast, our FedDisco brings no computation burden to the server and works with negligible computation overhead.

% ----------------------------------
\section{Experiments}

We show key experimental setups and results in this section. More details and results are in~\cref{sec:append_experiments}.

\subsection{Experimental Setup}

\textbf{Datasets.} We consider five image classification datasets to cover medical, natural and artificial scenarios, including HAM10000~\cite{tschandl2018ham10000}, CIFAR-10 \& CIFAR-100~\cite{cifar10}, CINIC-10~\cite{darlow2018cinic} and Fashion-MNIST~\cite{xiao2017fashion}; and AG News~\cite{zhang2015character}, a text classification dataset.

\textbf{Federated scenarios.} We consider two data distribution heterogeneous settings, termed NIID-1 and NIID-2. NIID-1 follows Dirichlet distribution~\cite{fedma,moon} $Dir_{\beta}$, where $\beta$ (default $0.5$) is an argument correlated with heterogeneity level. We consider $10$ clients for NIID-1. NIID-2 is a more heterogeneous setting consists of $5$ biased clients (each has data from ${C}/{5}$ categories~\cite{fedavg,fedprox}) and $1$ unbiased client (has all $C$ categories), where $C$ is the total category number.

\textbf{Implementation details.} The number of local epochs and batch size are $10$ and $64$, respectively. We run federated learning for $100$ rounds. We use ResNet18~\cite{resnet} for HAM10000, a simple CNN network for other image datasets and TextCNN~\cite{zhang2015sensitivity} for AG News. We use SGD optimizer with a $0.01$ learning rate. We evaluate the accuracy on the global testing set. We use KL-Divergence to measure the discrepancy.

\textbf{Baselines.} We compare FedDisco with eight representative baselines. Among these, 1) FedAvg~\cite{fedavg} is the pioneering FL method; 2) FedProx~\cite{fedprox}, SCAFFOLD~\cite{scaffold}, FedDyn~\cite{feddyn}, MOON~\cite{moon}, and FedDC~\cite{Gao_2022_CVPR} focus on local model adjustment; 3) FedAvgM~\cite{fedavgm} and FedNova~\cite{fednova} focus on global model adjustment. The tuned hyper-parameters are shown in~\cref{sec:append_hyper_parameter}.

\subsection{Main Results}

On four standard datasets and two types of heterogeneity, we compare FedDisco with state-of-the-art algorithms, show the modularity of FedDisco, and explore FedDisco's broader scope of applications.

\textbf{Performance: state-of-the-art accuracy.} We compare the accuracy of several state-of-the-art methods and our proposed FedDisco on multiple heterogeneous settings and datasets in Table~\ref{table:main}. Note that HAM10000 is an imbalanced dataset and thus is not applicable for NIID-2. The local training protocol used in FedDisco is FedDyn except SCAFFOLD for Fashion-MNIST. Experiments show that i) our proposed FedDisco consistently outperforms others across different settings, indicating the effectiveness of our proposed method; ii) FedDisco achieves significantly better on NIID-1 setting of HAM10000 ($\beta=5$), which is a more difficult task for its severe heterogeneity and imbalance.

\textbf{Modularity: improvements over baselines.} One key advantage of our proposed FedDisco is its modularity, that is, it can be a plug-and-play module in many existing FL methods to further improve their performance. Following the experiments in~\cref{table:main}, we report the accuracy of baselines combined with our Disco module and the accuracy difference (in parentheses) compared with baselines without Disco in~\cref{table:modularity} (see CIFAR-100 in~\cref{table:cifar100}). Experiments show that i) Disco consistently enhances the performance of state-of-the-art methods under different datasets and heterogeneity types; ii) for the most difficult task (i.e., HAM10000), FedDisco brings the largest performance improvement. Specifically, it achieves $19.8\%$ relative accuracy improvement over FedAvg~\cite{fedavg}.

\begin{table}[t]
\caption{Performance under partial client participation scenario without (w.o.) and with our proposed Disco module. FedDisco not only brings accuracy improvement, but also speeds up training.}
\label{table:partial}
\setlength\tabcolsep{2pt}
\vskip 0.15in
\begin{center}
\begin{small}
\begin{sc}
\begin{tabular}{ccccc}
\toprule
\multirow{2}{*}{Method} & \multicolumn{2}{c}{Acc ($\%$)} & \multicolumn{2}{c}{Rounds $\rightarrow 55\%$} \\
  & W.o. & With Disco & W.o. & With Disco \\
\midrule
FedAvg   & 54.27 & 61.59 ($\boldsymbol{\uparrow 7.32}$) & 58 & 20 ($\boldsymbol{\downarrow 65.52\%}$) \\
FedAvgM  & 57.75 & 60.65 ($\boldsymbol{\uparrow 2.90}$) & 36 & 20 ($\boldsymbol{\downarrow 44.44\%}$) \\
FedProx  & 55.50 & 60.19 ($\boldsymbol{\uparrow 4.69}$) & 44 & 20 ($\boldsymbol{\downarrow 54.55\%}$) \\
SCAFFOLD & 60.43 & 64.10 ($\boldsymbol{\uparrow 3.67}$) & 34 & 20 ($\boldsymbol{\downarrow 41.18\%}$) \\
FedDyn   & 59.44 & 62.53 ($\boldsymbol{\uparrow 3.09}$) & 33 & 27 ($\boldsymbol{\downarrow 18.18\%}$) \\
FedNova  & 54.11 & 58.13 ($\boldsymbol{\uparrow 4.02}$) & 61 & 24 ($\boldsymbol{\downarrow 60.66\%}$) \\
MOON     & 54.30 & 58.79 ($\boldsymbol{\uparrow 4.49}$) & 56 & 24 ($\boldsymbol{\downarrow 57.14\%}$) \\
FedDC    & 61.18 & 62.90 ($\boldsymbol{\uparrow 1.72}$) & 27 & 20 ($\boldsymbol{\downarrow 25.93\%}$) \\
\bottomrule
\end{tabular}
\end{sc}
\end{small}
\end{center}
\vskip -0.2in
\end{table}

\textbf{Applicability to partial client participation scenarios.} Partial client participation is a common scenario in cross-device FL applications where only a subset of clients are available in a specific FL round. To verify that our proposed FedDisco is applicable to this scenario, we sample $10$ out of $60$ clients for each round under NIID-2 on CIFAR-10, which consists of $50$ biased clients and $10$ unbiased clients. We report the averaged accuracy of last $10$ rounds and the number of rounds to reach target accuracy ($55\%$) in~\cref{table:partial}. We see that Disco module significantly i) brings performance gain to baselines (up to $7.32\%$); ii) reduces the communication cost to reach a target accuracy (up to $65.52\%$). These results indicate that FedDisco not only improves the accuracy but also speeds up the training process under this scenario.

\textbf{Applicability to text-modality scenarios.} To verify that our proposed FedDisco can also be applied to text-modality, we explore on a text classification dataset, AG News~\cite{tschandl2018ham10000} under full and partial participation scenarios. Results in~\cref{table:text} show that Disco still consistently improves the baselines on text modality and brings significant performance gain under partial client participation scenario.

\begin{table}[t]
\caption{Results on text classification dataset AG News~\cite{zhang2015character} under full and partial participation. FedDisco consistently brings accuracy improvement over baselines.}
\label{table:text}
\vskip 0.15in
\begin{center}
\begin{small}
\begin{sc}
\begin{tabular}{cccccc}
\toprule
\multirow{2}{*}{Disco?} & \multicolumn{2}{c}{Full} & \multicolumn{2}{c}{Partial} \\
& FedAvg & FedProx & FedAvg & FedProx \\
\midrule
$\times$ & 82.03 & 79.89 & 50.88 & 58.22 \\
$\surd$ & \textbf{84.09} & \textbf{80.81} & \textbf{57.71} & \textbf{62.86} \\
\bottomrule
\end{tabular}
\end{sc}
\end{small}
\end{center}
\vskip -0.2in
\end{table}

\textbf{Applicability to globally imbalanced scenarios.} Here, we verify that FedDisco is also capable of globally imbalanced category distribution scenario, that is the hypothetically aggregated global data is imbalanced. We firstly allocate CIFAR-10 to each category $c$ following an exponential distribution~\cite{cui2019class}: $n_c=n_1 \rho^{-\frac{c-1}{C-1}}$, where $\rho$ denotes the imbalance ratio and $C=10$ is the total category number, $n_1=5000$. The dataset is then distributed to $10$ clients as NIID-1. Note that $\rho=1$ denotes globally balanced, $\rho=20$ is the most imbalanced, where category $1$ has $5000$ samples while category $10$ only has $250$ samples. We consider two scenarios 1) the global category distribution is non-uniform while the distribution of test dataset is uniform; 2) the global category distribution and the distribution of test dataset are similar, and both non-uniform.

1) We conduct experiments on two typical baselines, FedAvg and FedDyn in~\cref{fig:gloabl_imb}. Experiments show that our FedDisco consistently improves the baseline regardless of the globally imbalance level; see more results in~\cref{append:imbalance}.

2) We conduct experiments under two scenarios ($\rho=10$ and $\rho=50$) and report the results in~\cref{table:imb_test}. Experiments show that FedDisco still performs the best when both global and test category distribution are non-uniform; see how to obtain global category distribution in a privacy-preserving way in~\cref{sec:append_obtain_distribution}.

\begin{table*}[t]
\caption{Performance on scenario where global category distribution and test distribution are similar but both non-uniform. Experiments show that FedDisco is also applicable to scenarios where the global and test distributions are non-uniform.}
\label{table:imb_test}
\setlength\tabcolsep{5pt}
\vskip 0.15in
\begin{center}
\begin{small}
\begin{sc}
\begin{tabular}{ccccccccccc}
\toprule
Method & FedAvg & FedAvgM & FedProx & SCAFFOLD & FedDyn & FedNova & MOON & FedDC & \textbf{FedDisco}\\
\midrule
$\rho=10$ & 69.78 & 69.25 & 69.27 & 71.48 & 69.29 & 68.44 & 68.29 & 70.00 & \textbf{71.77}\\
$\rho=50$ & 74.03 & 73.77 & 74.13 & 74.03 & 74.78 & 74.53 & 74.03 & 74.53 & \textbf{76.06}\\
\bottomrule
\end{tabular}
\end{sc}
\end{small}
\end{center}
\vskip -0.2in
\end{table*}

\begin{figure*}[t]
\vskip 0.2in
	\centering
	\subfigure[Globally Imbalance Level]{
		\includegraphics[width=0.23\textwidth]{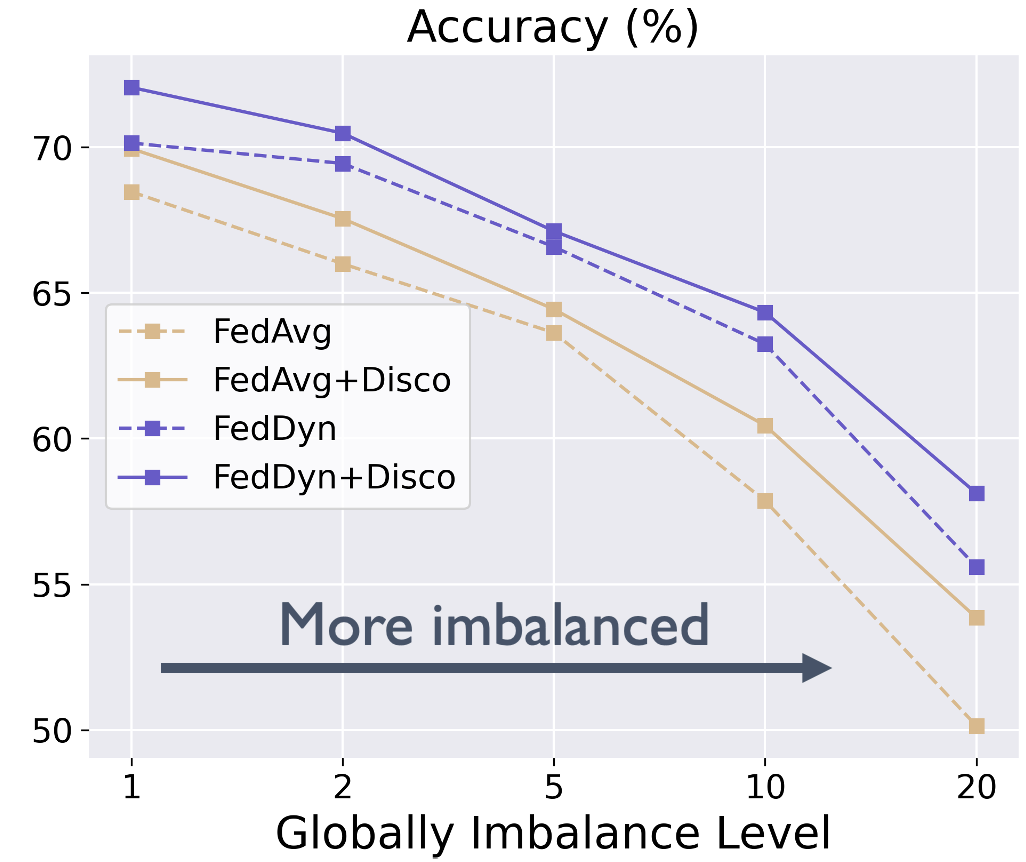}
        \label{fig:gloabl_imb}
	}
	\subfigure[Client Number]{
		\includegraphics[width=0.23\textwidth]{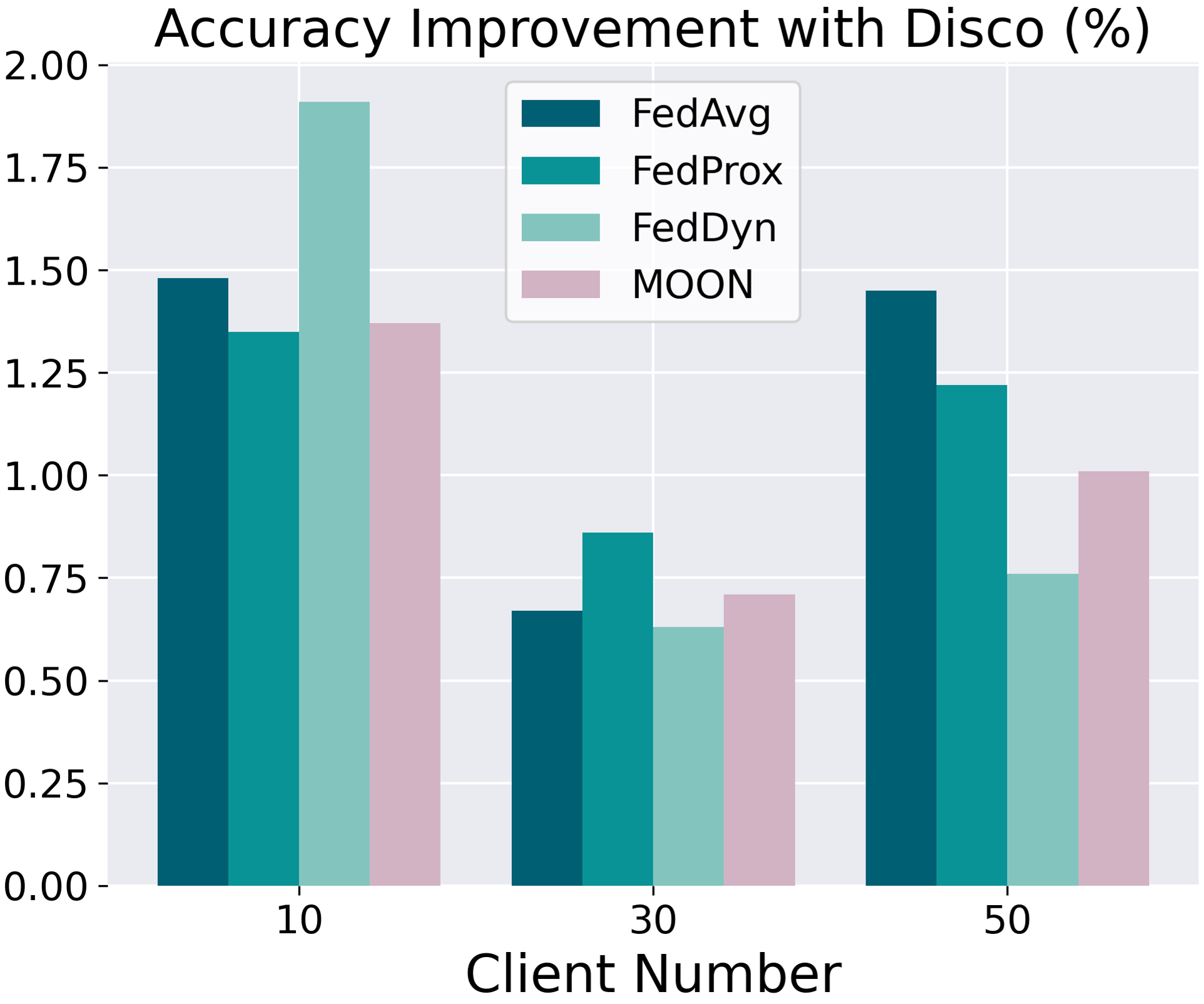}
        \label{fig:abl_client}
	}
	\subfigure[Heterogeneity Level]{
		\includegraphics[width=0.23\textwidth]{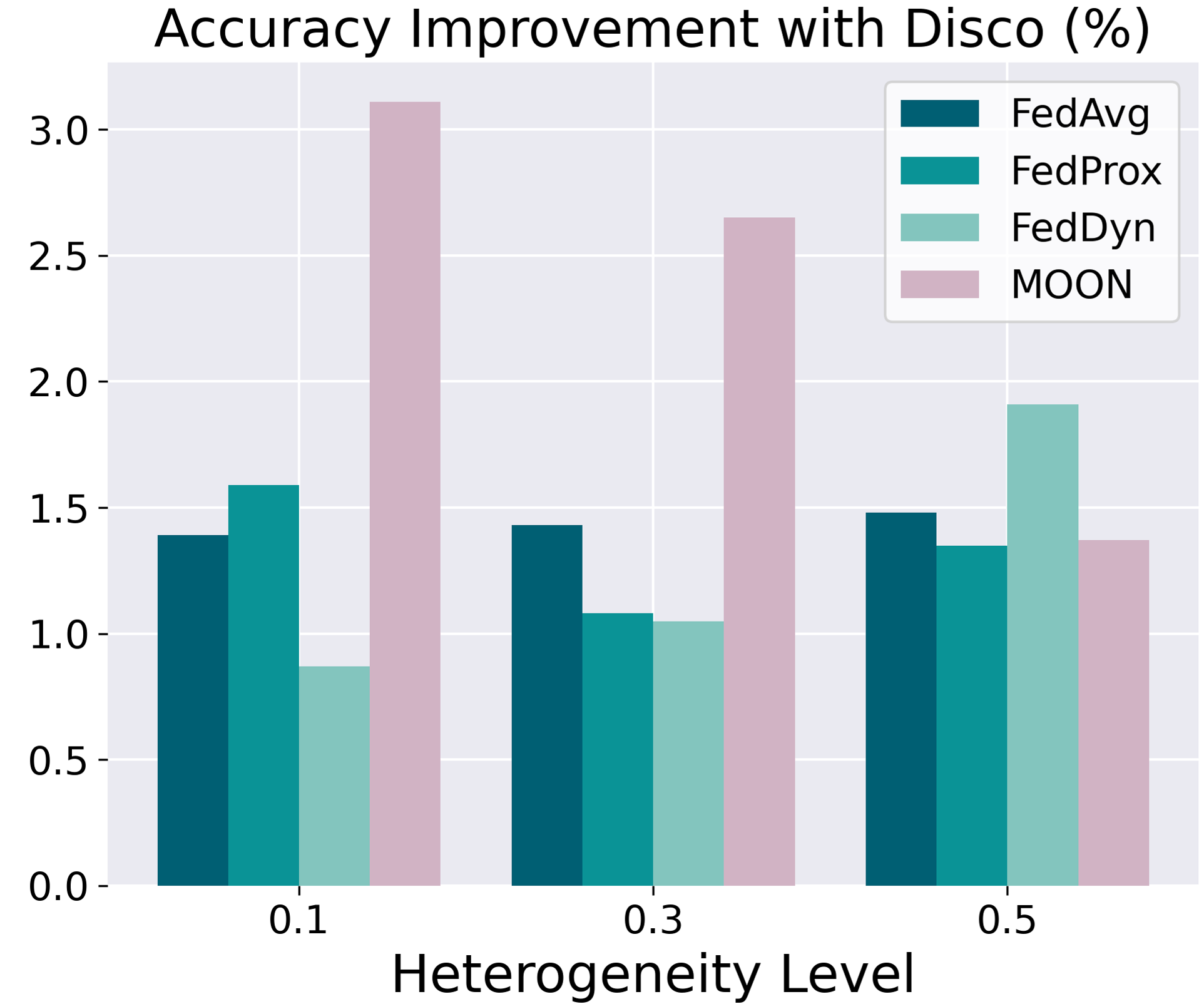}
        \label{fig:abl_hetero}
	}
	\subfigure[Number of Local Epoch]{
		\includegraphics[width=0.23\textwidth]{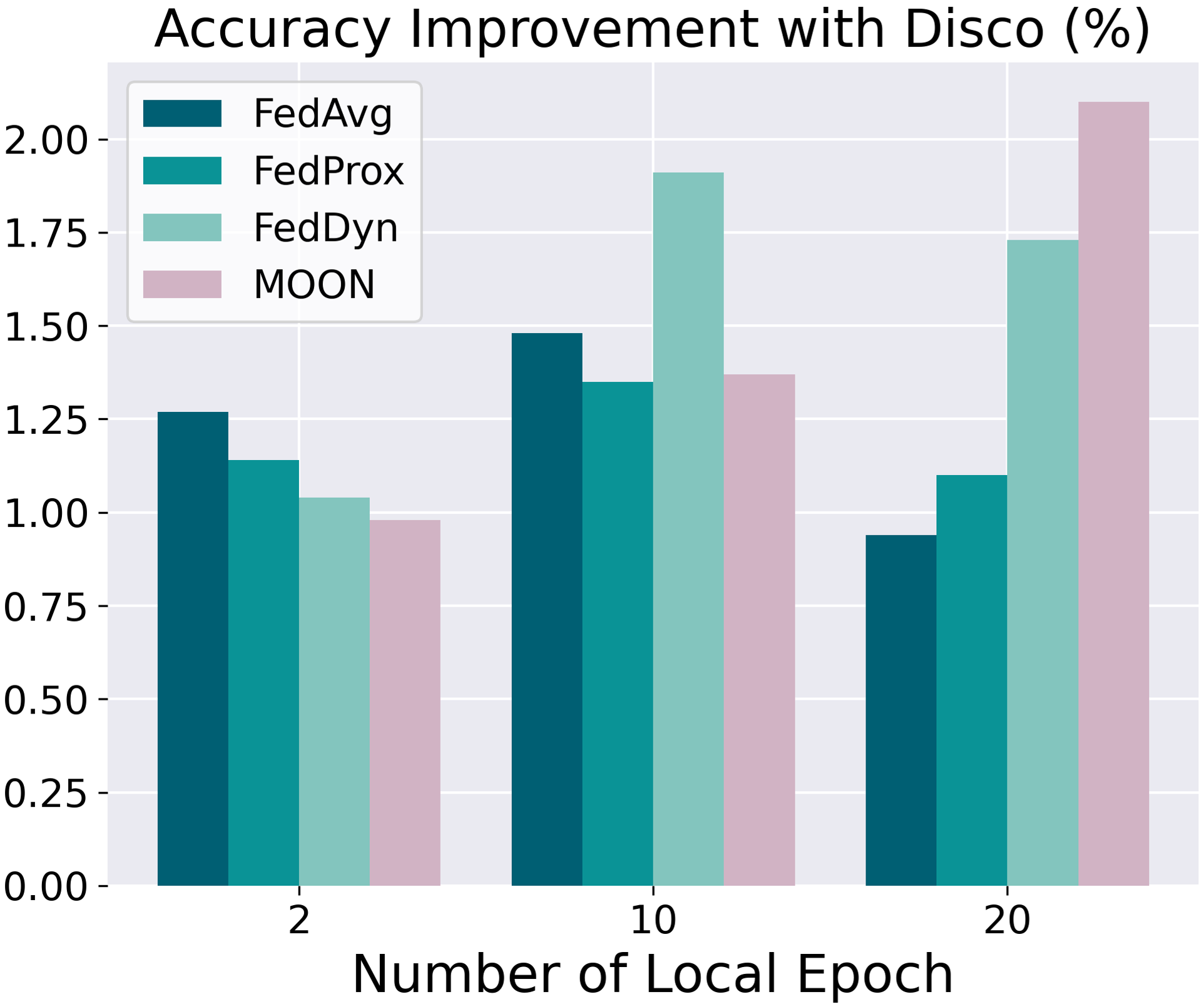}
        \label{fig:abl_epoch}
	}
    \vskip -0.1in
	\caption{Effects of four key FL arguments. Experiments show our Disco consistently brings performance improvement.}
	\label{figure:ablation}	
\vskip -0.2in
\end{figure*}

\subsection{Numerical Simulation of Reformulation}
\label{sec:numerical_simulation}

In~\cref{sec:theory}, we minimize the reformulated the upper bound in~\cref{theorem_1} for obtaining a concise expression, which benefits the practical utility as the tuning efforts are mitigated. Here, we conduct numerical simulation to examine the optimization process of the original form of upper bound in~\cref{theorem_1} and the reformulated form in~\cref{eq:reformulated}. We run $400$ steps of gradient descent on $\{p_k\}$ by minimizing the reformulated form, and record the resulting values of both reformulated and original form in~\cref{table:theory_exp}. Results show that minimizing the reformulated form benefit minimizing the original form, verifying the rationality of such reformulation during derivation.

\begin{table}[t]
\caption{Numerical simulation. Minimizing the reformulated form promotes minimizing the original form.}
\label{table:theory_exp}
\setlength\tabcolsep{2.5pt}
\vskip 0.15in
\begin{center}
\begin{small}
\begin{sc}
\begin{tabular}{cccccc}
\toprule
Step & 0 & 100 & 200 & 300 & 400 \\
\midrule
Reformulated & 0.0684 & 0.0670 & 0.0663 & 0.0660 & 0.0658 \\
Original & 0.0524 & 0.0513 & 0.0507 & 0.0504 & 0.0502 \\
\bottomrule
\end{tabular}
\end{sc}
\end{small}
\end{center}
\vskip -0.2in
\end{table}

\subsection{Ablation Study}

\textbf{Ease of hyper-parameter tuning.} We tune $a$ and $b$ in ~\eqref{eq:disco_real} under four settings to illustrate the ease of hyper-parameter tuning for our proposed FedDisco. For more comprehensive understanding, we consider multiple scenarios, including different baseline methods (FedAvg~\cite{fedavg} and MOON~\cite{moon}), heterogeneity types (NIID-1 and NIID-2) and datasets (CIFAR-10 and CINIC-10). We plot the accuracy difference brought by Disco for each group of hyper-parameters in Figure~\ref{figure:heatmaps}. By comparing three pairs: (a)\&(b), (b)\&(c) and (c)\&(d), we see that i) for a wide range ($a \in [0.2,0.7], b \in [0.05,0.4]$), our Disco brings performance improvement for most cases regardless of the baseline method, heterogeneity type and dataset; ii) generally, $a=0.4 \sim 0.6, b=0.1$ is a safe choice, which leads to stable and great performance.

\textbf{Effects of client number, heterogeneity level and local epoch.} We tune three key arguments in FL, including client number ($K \in \{10, 30, 50\}$), heterogeneity level (smaller value corresponds to more severe heterogeneity, $\beta \in \{0.1,0.3,0.5\}$) and number of local epoch $E \in \{2,10,20\}$, and show the accuracy improvement brought by Disco in~\cref{fig:abl_client}, \ref{fig:abl_hetero}, and \ref{fig:abl_epoch}, respectively. Experiments show that our proposed Disco consistently brings performance improvement across different FL arguments.

\begin{table}[t]
\caption{Effects of discrepancy metric. FedDisco is robust to various choices of discrepancy metrics.}
\label{table:metric}
\vskip 0.15in
\begin{center}
\begin{small}
\begin{sc}
\begin{tabular}{ccccc}
\toprule
Metric & L1 & L2 & Cosine & KL-Divergence \\
\midrule
Acc & 69.43 & 69.99 & 69.90 & 70.05 \\
\bottomrule
\end{tabular}
\end{sc}
\end{small}
\end{center}
\vskip -0.2in
\end{table}

\textbf{Effects of different discrepancy metrics.} Unless specified, we use the KL-Divergence throughout the paper. Though, Table~\ref{table:metric} compares four discrepancy metrics under NIID-1 on CIFAR-10, including L1\&L2 norm, cosine similarity and KL-Divergence. Experiments show that FedDisco with these metrics achieve similar performance, indicating its robustness to different discrepancy metrics.

\section{Conclusions}
This paper focuses on data heterogeneity issue in FL. Through empirical and theoretical explorations, we find that conventional dataset-size-based aggregation manner could be far from optimal. Addressing these, we introduce a discrepancy value as a complementary indicator and propose FedDisco, a novel FL algorithm that assigns larger aggregation weight to client with larger dataset size and smaller discrepancy. FedDisco introduces negligible computation and communication cost, and can be easily incorporated with many methods. Experiments show that FedDisco consistently achieves state-of-the-art performances. 

Though this work mainly explores category-level heterogeneity, we may extend the idea of discrepancy-aware collaboration to other types of heterogeneity, such as feature-level heterogeneity for classification task and mask-level heterogeneity for segmentation task.

\section*{Acknowledgements}

This research is supported by the National Key R\&D Program of China under Grant 2021ZD0112801, NSFC under Grant 62171276 and the Science and Technology Commission of Shanghai Municipal under Grant 21511100900 and 22DZ2229005.


\bibliography{example_paper}
\bibliographystyle{icml2023}

%%%%%%%%%%%%%%%%%%%%%%%%%%%%%%%%%%%%%%%%%%%%%%%%%%%%%%%%%%%%%%%%%%%%%%%%%%%%%%%
%%%%%%%%%%%%%%%%%%%%%%%%%%%%%%%%%%%%%%%%%%%%%%%%%%%%%%%%%%%%%%%%%%%%%%%%%%%%%%%
% APPENDIX
%%%%%%%%%%%%%%%%%%%%%%%%%%%%%%%%%%%%%%%%%%%%%%%%%%%%%%%%%%%%%%%%%%%%%%%%%%%%%%%
%%%%%%%%%%%%%%%%%%%%%%%%%%%%%%%%%%%%%%%%%%%%%%%%%%%%%%%%%%%%%%%%%%%%%%%%%%%%%%%
\newpage
\appendix
\onecolumn

\section{Methodology}
\label{append:method}

\subsection{Complementary Description}

\textbf{Notation table.} For convenience, we provide a detailed notation descriptions in~\cref{table:notation}.

\begin{table}[t]
\caption{Notation descriptions.}
\label{table:notation}
\vskip 0.15in
\begin{center}
\begin{small}
\begin{tabular}{l|l}
\toprule
Notation & Description \\ \midrule
$K$ & The total client number in FL system \\
$\tau$ & The number of SGD steps during local model training for each round\\
$\mathcal{B}_k$ & Local private dataset of client $k$ \\
$\mathbf{w}^{(t, r)}_k$ & Local model of client $k$ at round $t$ and iteration $r$ \\
$\mathbf{w}^{(t,0)}$ & Global model at round $t$ \\
$\mathbf{D}_k$ & Local category distribution of client $k$ \\
$\mathbf{D}_{k,c}$ & The $c$-th element of $\mathbf{D}_k$ \\
$\mathbf{T}$ & Global category distribution \\
$n_k$ & The relative dataset size of client $k$ \\
$p_k$ & The aggregation weight of client $k$ \\
$d_k$ & The discrepancy value of client $k$ \\
$F_k(\mathbf{w})$ & Local objective of client $k$ \\
$F(\mathbf{w})$ & Global objective of FL \\
\bottomrule
\end{tabular}
\end{small}
\end{center}
\vskip -0.2in
\end{table}

\textbf{Algorithm table.} We provide the overall algorithm in~\cref{alg:feddisco}.

\begin{algorithm}[t]
   \caption{\colorbox{red!15}{FedDisco: Federated Learning with Discrepancy-Aware Collaboration}}
   \label{alg:feddisco}
\begin{algorithmic}
   \STATE {\bfseries Initialization:} Global model $\mathbf{w}^{(0,0)}$
   \FOR{$k=0$ {\bfseries to} $K-1$}
        \STATE \colorbox{red!15}{Client sends discrepancy value $d_k$ to server}
   \ENDFOR
   \STATE \colorbox{red!15}{Server computes the aggregation weight $p_k$ according to~\cref{eq:disco_real}}
   \FOR{$t=0$ {\bfseries to} $T-1$}
        \STATE Server sends global model $\mathbf{w}^{(t,0)}$ to each client
        \FOR{$k=0$ {\bfseries to} $K-1$}
            \STATE $\mathbf{w}_k^{(t,\tau)} \leftarrow $ local model training for $\tau$ steps of SGD
            \STATE Client sends local model $\mathbf{w}_k^{(t,\tau)}$ to server
        \ENDFOR
        \STATE \colorbox{red!15}{Server aggregates local models $\mathbf{w}^{(t+1,0)}=\sum_{k=1}^K p_k \mathbf{w}_k^{(t,\tau)}$}
   \ENDFOR
\end{algorithmic}
\end{algorithm}

\subsection{Discussions}
\label{append:discussions}

\textbf{Connection with multi-task learning method, Nash-MTL~\cite{nash-mtl}.} Nash-MTL focuses on aggregation weights of multiple tasks and we focus on aggregation weights of multiple clients. We will cite this paper in the revision. However, there are two major differences between Nash-MTL and our work. 1) Nash-MTL focuses on multi-task learning in a centralized manner while we focus on federated learning in a distributed manner.
2) The aggregation weights in Nash-MTL is learned through minimizing a pre-defined problem to search for Nash bargaining solution, which focuses on pair-wise gradient relationships among multiple tasks. In comparison, our aggregation weights is directly computed based on dataset size and discrepancy level, whose design is guided by our empirical and theoretical observation.

\subsection{Obtaining Global Category Distribution in A Privacy-Preserving Manner}
\label{sec:append_obtain_distribution}

In~\cref{sec:method}, we regard the target distribution $\mathbf{T}$ as uniform since we want to emphasize category-level fairness. However, our method is also applicable in scenarios where the global category distribution and test distribution are both non-uniform; see results in~\cref{table:imb_test}. Here, we show that we can obtain the global category distribution in a privacy-preserving way.

Specifically, we can send each client's category distribution $\mathbf{D}_k$ (a vector) to the server using Secure Aggregation~\cite{secure_aggregation} technique such that the server knows the actual global category distribution without knowing each client's actual category distribution. Then, this actual category distribution can be sent to each client and the discrepancy can be calculated (this process is efficient as it only requires once). As a simple example for the secure aggregation process, the distribution vectors of Client 1, 2 are $\mathbf{D}_1$ and $\mathbf{D}_2$. Client 1 adds an arbitrary noise vector $\mathbf{a}$ to $\mathbf{D}_1$ to obtain $\mathbf{D}_1+\mathbf{a}$; while Client 2 substitutes $\mathbf{a}$ from $\mathbf{D}_2$ to obtain $\mathbf{D}_2-\mathbf{a}$. These two transformed vectors are then sent to the server, where the server knows the sum of global category distribution without knowing the exact distribution of each client: $\mathbf{D}_1+\mathbf{a} + \mathbf{D}_2-\mathbf{a} = \mathbf{D}_1 + \mathbf{D}_2$.

\section{Experiments}
\label{sec:append_experiments}

\subsection{Implementation details}

\subsubsection{Environments}

We run all methods by using Pytorch framework~\cite{paszke2019pytorch} on a single NVIDIA GTX 3090 GPU. The memory occupation ranges from 2065MB to 4413MB for diverse datasets and methods.

\subsubsection{Datasets}

We use five image classification datasets, which cover medical, natural and artificial images. For medical image classification dataset, we consider HAM10000~\cite{codella2019skin,tschandl2018ham10000}, a $7$ category classification task for pigmented skin lesions. For natual image classification dataset, we consider CIFAR-10, CIFAR-100 and CINIC-10~\cite{cifar10,darlow2018cinic}, all of which are classification tasks for natural objects with $10$, $100$ and $10$ categories, respectively. For artificial image classification dataset, we consider Fashion-MNIST~\cite{xiao2017fashion}, a $10$ category classification task for clothing. All these four public datasets can be downloaded online. Note that for HAM10000, we hold out a uniform testing set by allocating $30$ samples for each category. We use one text classification dataset, AG News~\cite{zhang2015character}, which is a $4$ classification task.
% \added{XMK: As the labels of original testing set are not accessible, we use the validation set of ISIC2018 as our testing set. For the categories with less than 30 samples, we move corresponding number of samples from training set to testing set. For others, we randomly select 30 samples.}

\subsubsection{Heterogeneity Level}
Here, we illustrate the data distribution over categories and clients under NIID-1 setting of four datasets. As mentioned in the main paper, NIID-1 denotes the setting of Dirichlet distribution $Dir_{\beta}$. We set $\beta=0.5$ for all datasets except HAM10000 and CIFAR-100. For HAM10000, all methods fail at NIID-1 scenario when $\beta$ is too small since HAM10000 is an severely imbalanced dataset. Thus, we choose a moderated $\beta=5.0$. We also choose $\beta=5.0$ for CIFAR-100. We plot the NIID-1 data distribution over categories and clients on four datasets in Figure (\ref{figure:data_distribution}). Note that HAM10000 is globally imbalanced and it has the largest number of samples in class $0$. For AG News, we partition the dataset to $5$ and $50$ clients for full and partial participation scenarios, where $80\%$ biased clients have data samples from $2$ categories and the other $20\%$ uniform clients have data samples from $4$ categories.

\begin{figure}[htbp] 
	\centering
	\subfigure[CIFAR-10]{
		\includegraphics[width=0.23\textwidth]{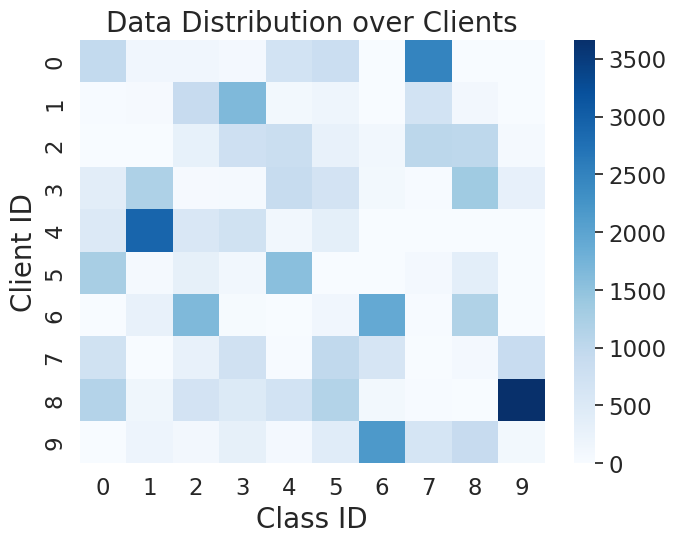}
	}
	\subfigure[CINIC-10]{
		\includegraphics[width=0.23\textwidth]{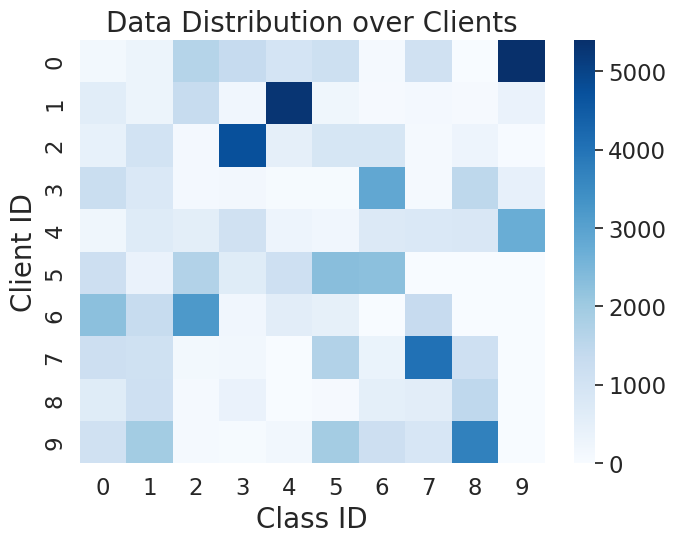}
	}
	\subfigure[Fashion-MNIST]{
		\includegraphics[width=0.23\textwidth]{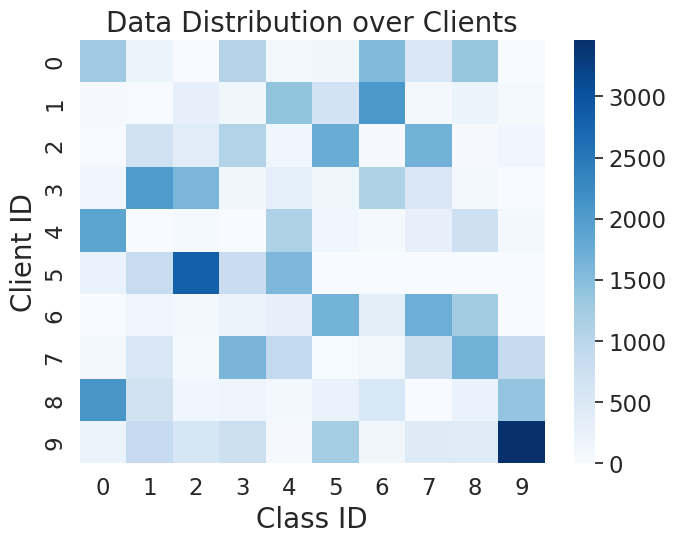}
	}
	\subfigure[HAM10000]{
		\includegraphics[width=0.23\textwidth]{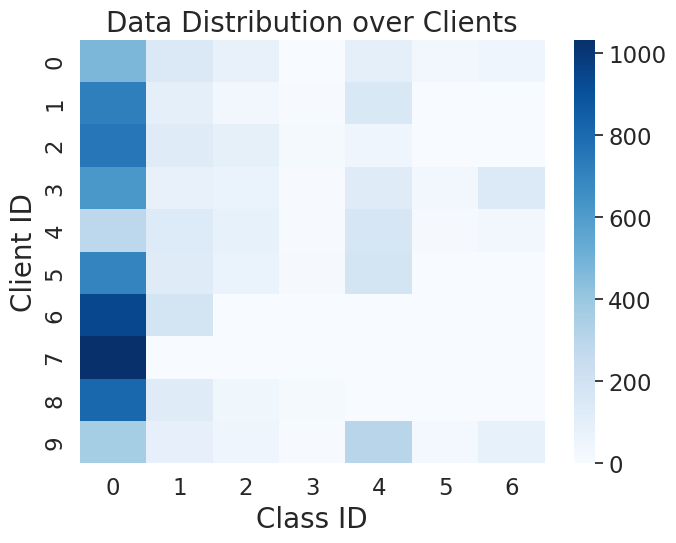}
	}
	\caption{Data distribution over categories and clients under NIID-1 setting.}
	\label{figure:data_distribution}	
\end{figure}

\subsubsection{Models}

For the CIFAR-10, CINIC-10 and Fashion-MNIST datasets, we use a simple CNN network as~\cite{moon,nofear}. The network is sequentially consists of: $5 \times 5$ convolution layer, max-pooling layer, $5 \times 5$ convolution layer, three fully-connected layer with hidden size of $120$, $84$ and $10$ respectively. For the HAM10000 and CIFAR-100 dataset, we use ResNet18~\cite{resnet} in Pytorch library. We replace the first $7 \times 7$ convolution layer with a $3 \times 3$ convolution layer and eliminate the first pooling layer. We also replace the batch normalization layer with group normalization layer as~\cite{feddyn}. For AG News~\cite{zhang2015character}, we use TextCNN model~\cite{zhang2015sensitivity} with a $32$ hidden dimension.

\subsubsection{Hyper-parameters}
\label{sec:append_hyper_parameter}
For all methods, we tune the hyper-parameter in a reasonable range and report the highest accuracy in the paper. For FedProx~\cite{fedprox}, we tune the hyper-parameter $\mu$ from $\{0.001, 0.01, 0.1, 1.0\}$. For FedAvgM~\cite{fedavgm}, we tune the hyper-parameter $\beta$ from $\{0.3, 0.5, 0.7, 0.9\}$. For FedDyn~\cite{feddyn}, we tune the hyper-parameter $\alpha$ from $\{0.001, 0.01, 0.1, 1.0\}$. For MOON~\cite{moon}, we tune the hyper-parameter $\mu$ from $\{0.01, 0.1, 0.5, 1.0, 5.0\}$. For FedDC~\cite{Gao_2022_CVPR}, we tune the hyper-parameter $\alpha$ from $\{0.001, 0.01, 0.1, 1.0\}$.

The tuned best hyper-parameter for these methods are: $\mu=0.01$ in FedProx~\cite{fedprox}, $\beta=0.5$ in FedAvgM~\cite{fedavgm}, $\alpha=0.01$ in FedDyn~\cite{feddyn}, $\mu=0.1$ in MOON~\cite{moon}, $\alpha=0.01$ in FedDC~\cite{Gao_2022_CVPR}.

\subsection{Globally Imbalanced Scenario: Accuracy and Fairness}
\label{append:imbalance}

We show that our proposed FedDisco is also capable of globally imbalanced category distribution scenario, that is the hypothetically aggregated global data is imbalanced. In the main paper, we show the performance of FedAvg~\cite{fedavg} and FedDyn~\cite{feddyn} with and without Disco for different globally imbalance level ($\rho={1,2,5,10,20}$) in Figure 5 (a).

We firstly visualize the hypothetically aggregated global data distribution in Figure~\ref{figure:imbalanced_distribution}, where x-axis denotes the category and y-axis denotes the number of data samples of the corresponding category. We plot the data distribution over categories for different globally imbalance levels ($\rho$). $\rho=1$ denotes balanced situation and $\rho=20$ denotes the most severe imbalanced situation. Our experiments in the Figure 5 (a) of the main paper show that our FedDisco works for all these globally imbalance levels.

\begin{figure}[htbp]  
	\centering
	\subfigure[Balanced: $\rho=1$]{
		\includegraphics[width=0.18\textwidth]{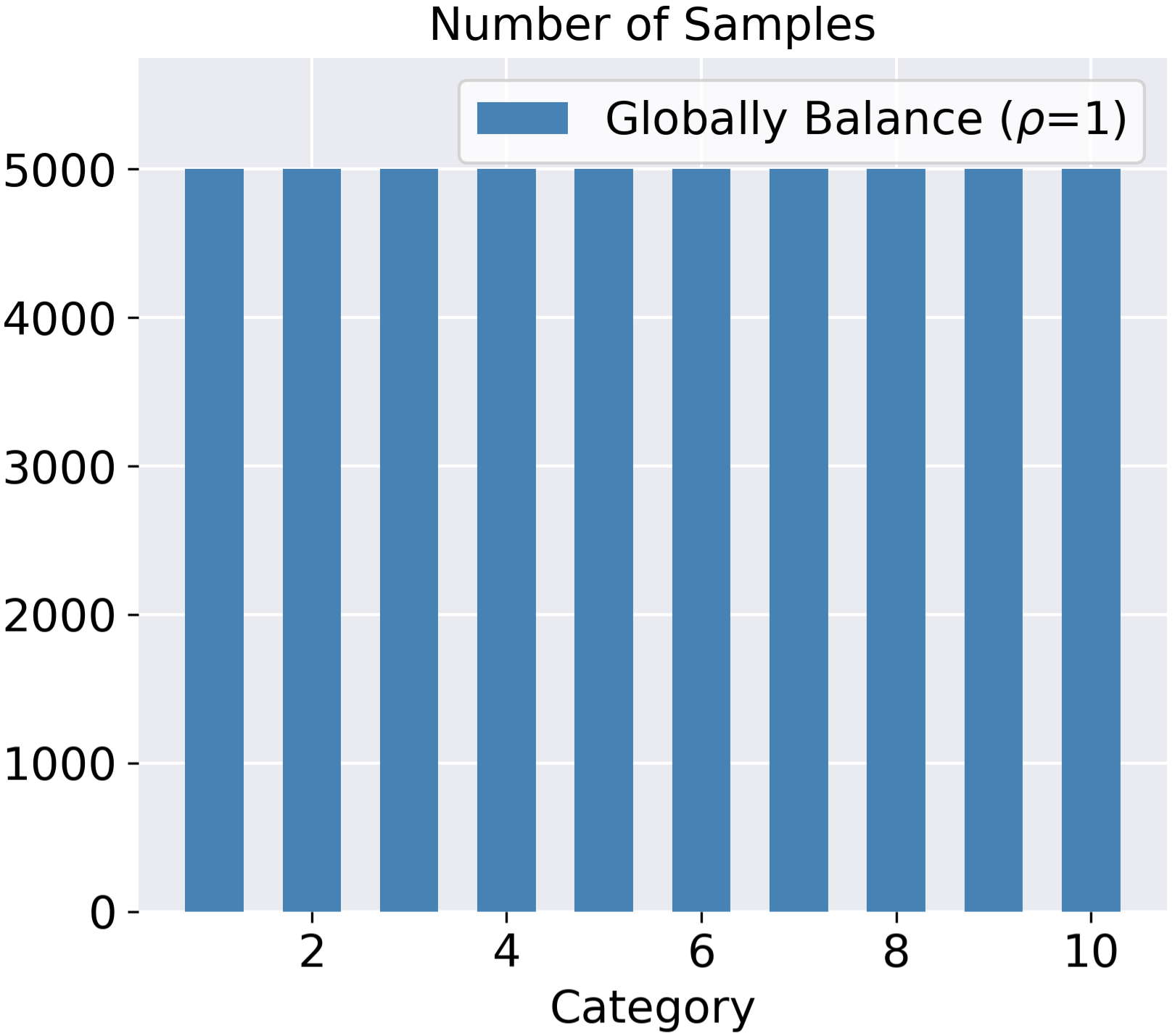}
	}
	\subfigure[Imbalanced: $\rho=2$]{
		\includegraphics[width=0.18\textwidth]{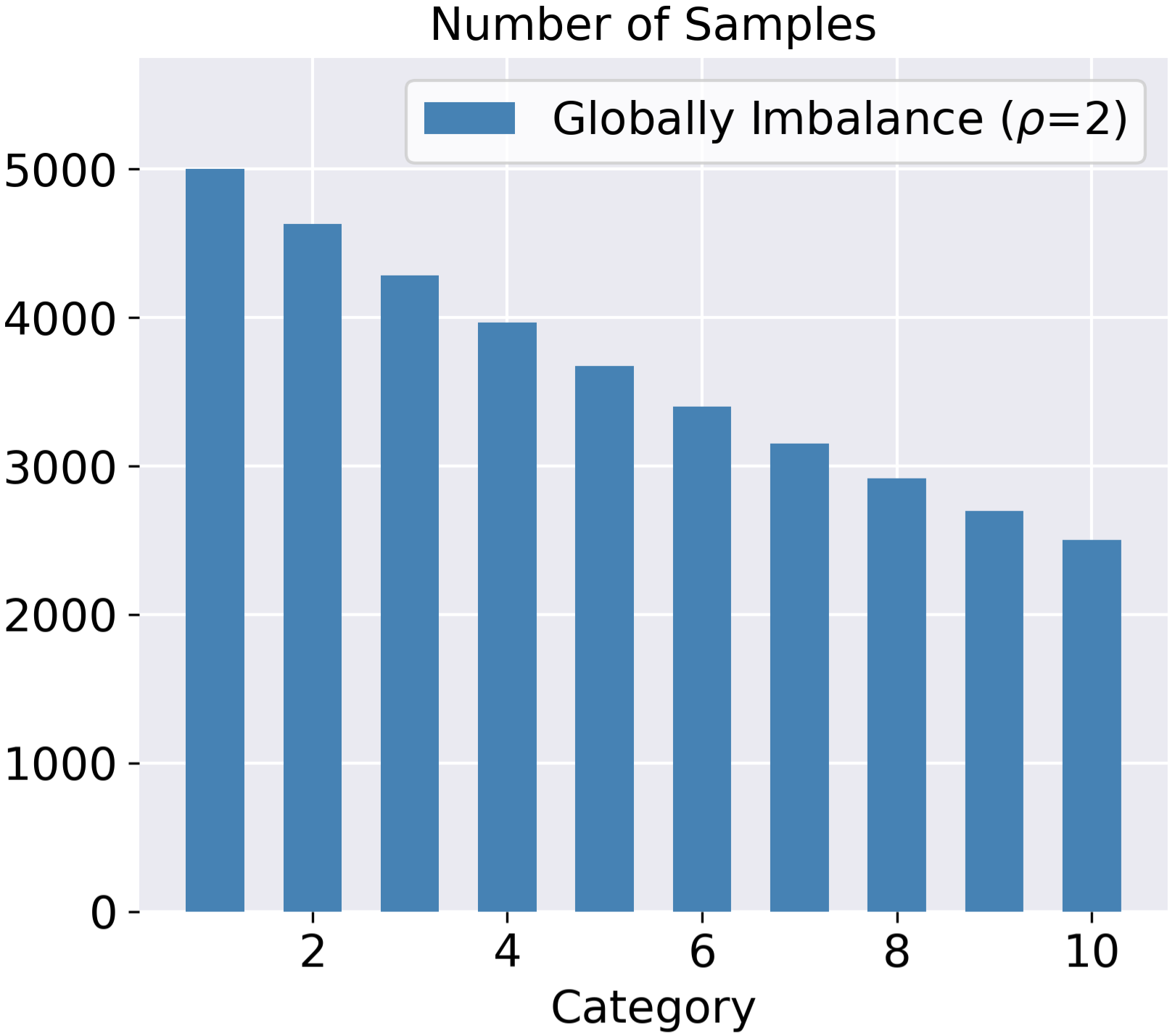}
	}
	\subfigure[Imbalanced: $\rho=5$]{
		\includegraphics[width=0.18\textwidth]{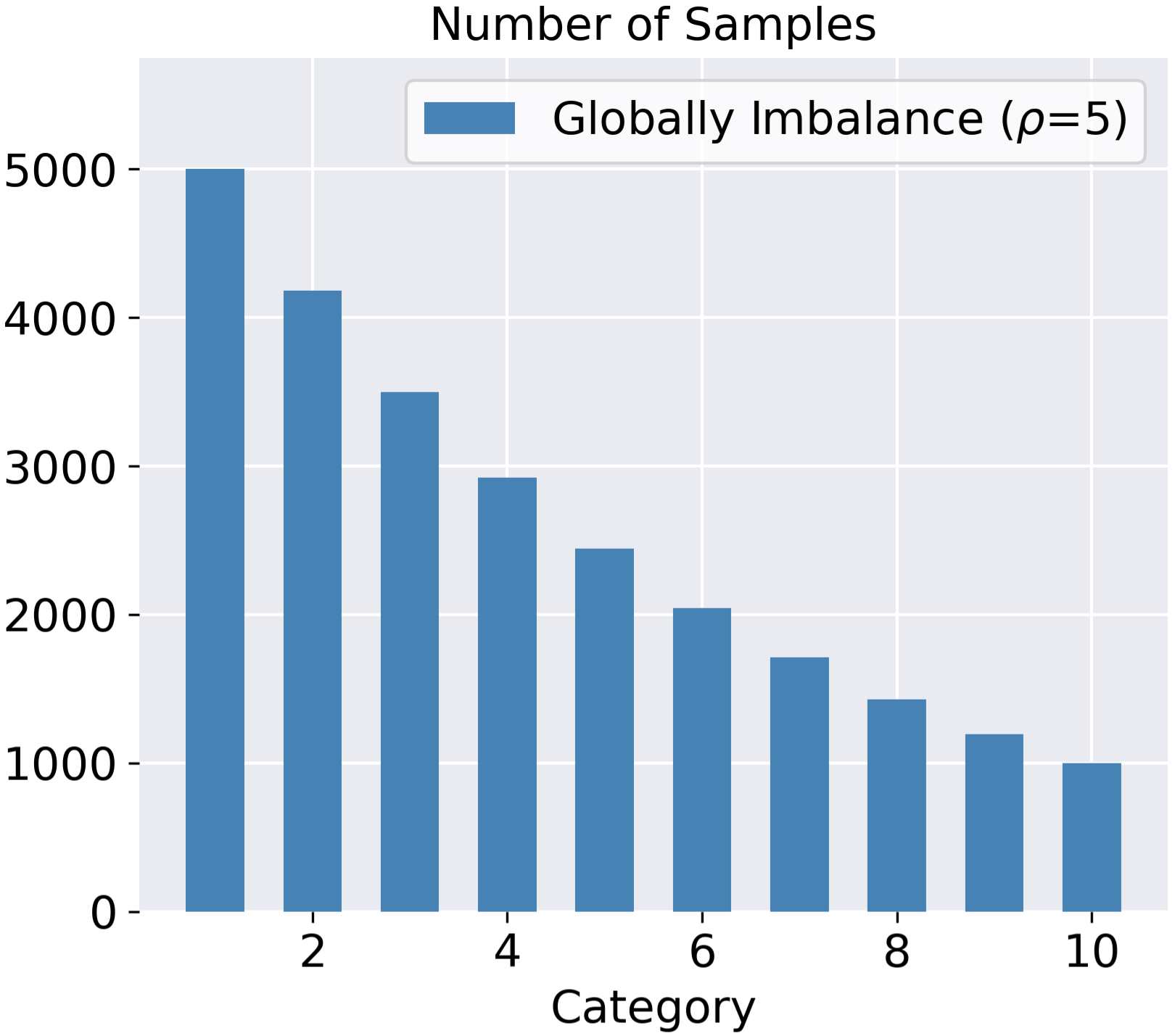}
	}
	\subfigure[Imbalanced: $\rho=10$]{
		\includegraphics[width=0.18\textwidth]{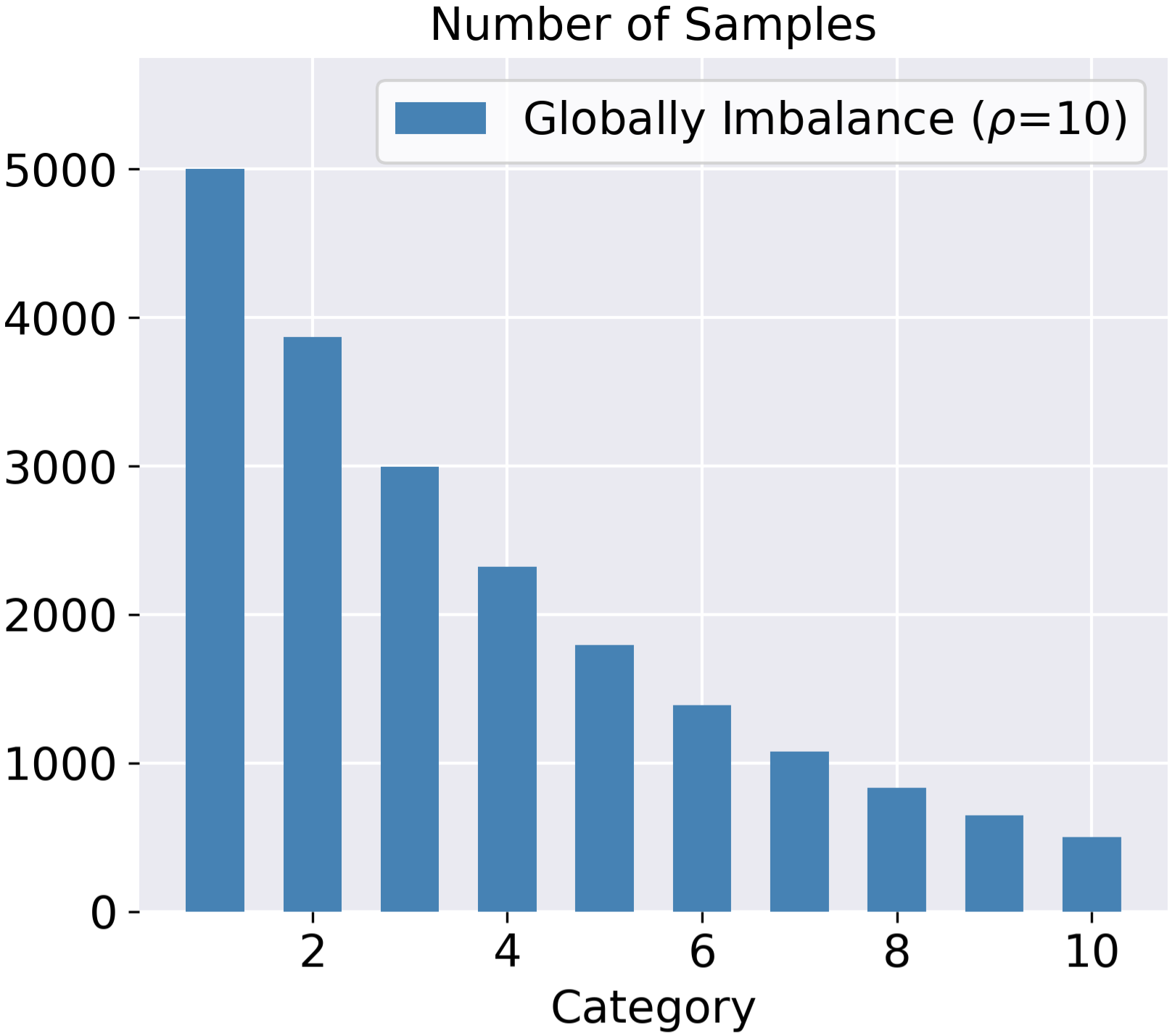}
	}
	\subfigure[Imbalanced: $\rho=20$]{
		\includegraphics[width=0.18\textwidth]{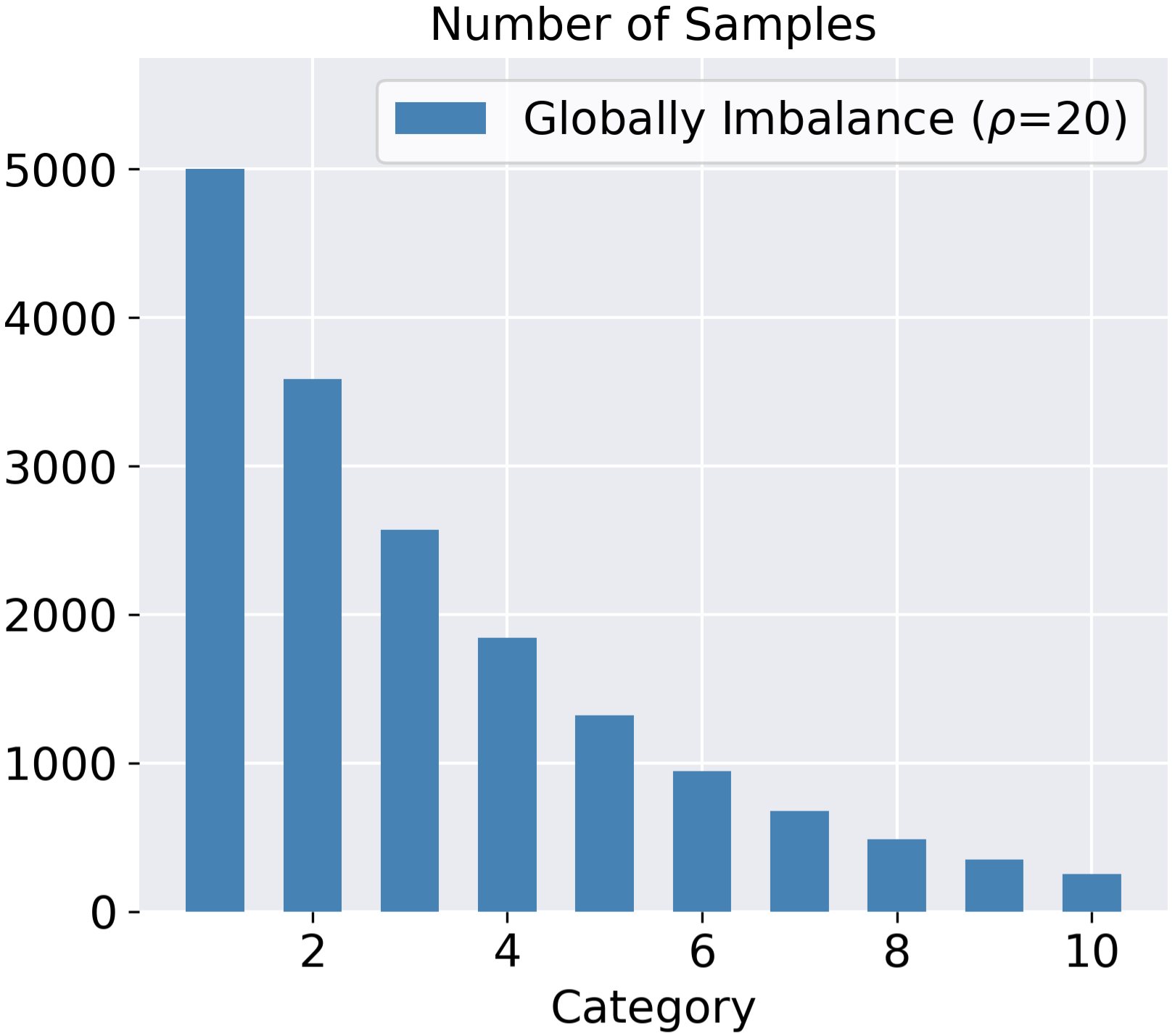}
	}
	\caption{Hypothetically aggregated global data distribution over categories for different globally imbalance level ($\rho$). The globally imbalance level increases from the left to the right, where $\rho=1$ denotes balanced situation and $\rho=20$ denotes the most severe imbalanced situation.}
	\label{figure:imbalanced_distribution}	
\end{figure}

Next, we provide the performance of more baselines when $\rho=10$ in Table  \ref{table:global_imbalanced}. Experiments show that our proposed FedDisco consistently improves over baselines under this globally imbalanced scenario, indicating its applicability to this scenario.

\begin{table}[t]
\caption{Classification accuracy ($\%$) comparisons under globally imbalanced dataset scenario ($\rho=10$). We highlight the \textbf{best} performance and \textbf{\emph{second-best}} performance. SCAFFOLD and FedDyn performs the best while SCAFFOLD requires twice communication costs. Experiments show our proposed FedDisco consistently improves the performances of baselines, indicating FedDisco's applicability to this scenario.}
\label{table:global_imbalanced}
\vskip 0.15in
\begin{center}
\begin{small}
\begin{sc}
\begin{tabular}{ccccccccc}
\toprule
Method & FedAvg & FedAvgM & FedProx & SCAFFOLD & FedDyn & FedNova & MOON & FedDC
\\ \midrule
Without Disco & 57.86 & 57.53 & 57.78 & 63.78 & 63.24 & 59.64 & 57.74 & 61.74 \\
With Disco & 60.45 & 60.05 & 60.33 & \textbf{65.13} & \textbf{\emph{64.33}} & 60.77 & 59.46 & 63.13\\ \bottomrule
\end{tabular}
\end{sc}
\end{small}
\end{center}
\vskip -0.2in
\end{table}

We also explore the performance of global model on each individual category in detail in Figure~\ref{figure:fairness}. Here, we use FedAvg~\cite{fedavg} as an example. Following~\cite{long_tail}, we define and evaluate on three subsets: Head (category 1 - 4), Middle (category 5 - 7) and Tail (category 8 - 10). Additionally, we report the averaged accuracy over all categories and standard deviation across categories. We see that i) FedDisco achieves comparable performance (difference within $0.75\%$) on Head and Middle classes and significantly higher performance (nearly $8\%$ improvement) on Tail classes. This suggests that FedAvg is severally biased to Head classes while FedDisco can mitigate this bias. ii) FedDisco achieves higher averaged accuracy with smaller standard deviation, which means FedDisco can simultaneous enhance overall performance and promote fairness across categories.

\begin{figure}[htbp]
\centering
\includegraphics[width=0.4\columnwidth]{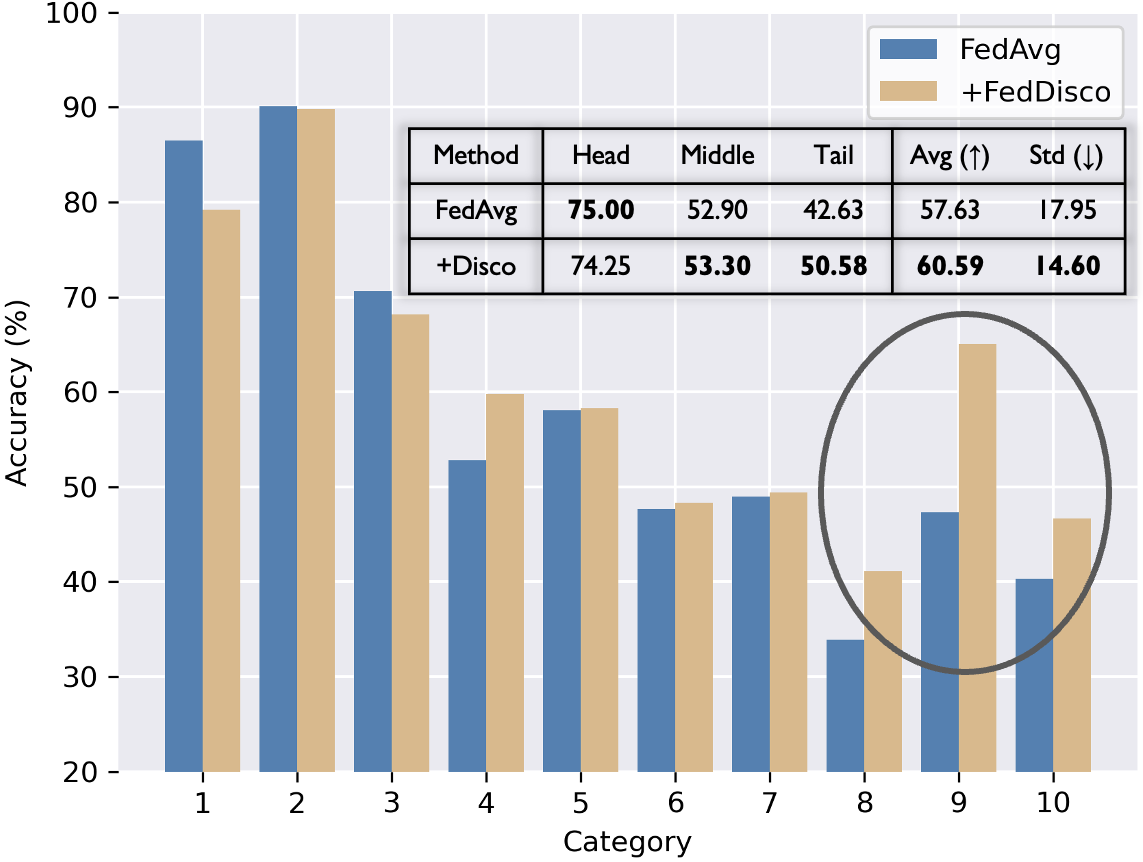}
\caption{Accuracy comparison of each category on CIFAR-10. FedDisco significantly enhance the performance of Tail classes without severely affecting the Head classes. FedDisco simultaneously enhance overall performance (Higher Avg) and fairness (Lower Std) over FedAvg.}
\label{figure:fairness}
\end{figure}

\subsection{Partial Client Participation Scenario: Accuracy, Stability and Training Speed}

In practice, clients may participate only when they are in reliable power and Wi-Fi condition, indicating that only a subset of clients participate in each round (partial client participation). We conduct experiments on CIFAR-10~\cite{cifar10} where there are $40$ biased clients and $10$ unbiased clients. In each round, we randomly sample $10$ clients to participate.

We show results from two aspects in Table \ref{table:partial_3}, including the mean and standard deviation (std) of accuracy of last $10$ rounds and rounds required to reach target accuracy ($55\%$). Experiments show that methods with our Disco achieve i) consistently higher mean and smaller std of accuracy, indicating better and more stable performance over rounds, ii) fewer rounds to achieve target accuracy, indicating faster training speed, which is friendly to save computation and communication cost. \emph{Specifically, for the basic method FedAvg, FedAvg with Disco achieves $7.32\%$ accuracy improvement and requires $38$ less rounds to achieve the target accuracy, which saves $65.52\%$ training time and communication cost.}

\begin{table}[t]
\caption{Performance under partial client participation scenario. Mean $\pm$ std is evaluated over last $10$ rounds. Rounds: rounds to reach target accuracy ($55\%$). Higher mean, lower std and smaller rounds represent better performance. We highlight the difference of mean accuracy brought by Disco in parentheses. We also highlight the reduced number of rounds together with the proportion of reduction brought by Disco in parentheses. Methods with Disco achieve higher accuracy, more stable performance and faster training. Specifically, FedAvg with Disco achieves $7.32\%$ accuracy improvement and requires $38$ less rounds to achieve the target accuracy, which saves $65.52\%$ training time and communication cost.}
\label{table:partial_3}
\vskip 0.15in
\begin{center}
\begin{small}
\begin{sc}

\begin{tabular}{ccccc}
\toprule
\multirow{2}{*}{Method} & \multicolumn{2}{c}{Mean $\pm$ std ($\%$)} & \multicolumn{2}{c}{Rounds} \\ 
  & Without Disco & With Disco & Without Disco & With Disco\\ \midrule
FedAvg   & 54.27 $\pm 3.44$ & 61.59 $\pm 1.84$ ($\boldsymbol{\uparrow 7.32}$) & 58 & 20 ($\boldsymbol{\downarrow 38, \downarrow 65.52\%}$) \\
FedAvgM  & 57.75 $\pm 1.92$ & 60.65 $\pm 1.32$ ($\boldsymbol{\uparrow 2.90}$) & 36 & 20 ($\boldsymbol{\downarrow 16, \downarrow 44.44\%}$) \\
FedProx  & 55.50 $\pm 3.02$ & 60.19 $\pm 2.47$ ($\boldsymbol{\uparrow 4.69}$) & 44 & 20 ($\boldsymbol{\downarrow 24, \downarrow 54.55\%}$) \\
SCAFFOLD & 60.43 $\pm 1.96$ & 64.10 $\pm 0.76$ ($\boldsymbol{\uparrow 3.67}$) & 34 & 20 ($\boldsymbol{\downarrow 14, \downarrow 41.18\%}$) \\
FedDyn   & 59.44 $\pm 2.59$ & 62.53 $\pm 2.34$ ($\boldsymbol{\uparrow 3.09}$) & 33 & 27 ($\boldsymbol{\downarrow 06, \downarrow 18.18\%}$) \\
FedNova  & 54.11 $\pm 3.83$ & 58.13 $\pm 2.72$ ($\boldsymbol{\uparrow 4.02}$) & 61 & 24 ($\boldsymbol{\downarrow 37, \downarrow 60.66\%}$) \\
MOON     & 54.30 $\pm 3.97$ & 58.79 $\pm 3.07$ ($\boldsymbol{\uparrow 4.49}$) & 56 & 24 ($\boldsymbol{\downarrow 32, \downarrow 57.14\%}$) \\ \bottomrule
\end{tabular}
\end{sc}
\end{small}
\end{center}
\vskip -0.2in
\end{table}

\subsection{Experimental results on CIFAR-100}
\label{append:cifar100}
We provide the results of several baselines on CIFAR-100~\cite{cifar10} under NIID-1 setting in Table \ref{table:cifar100}. Experiments show that our proposed FedDisco consistently brings gains on CIFAR-100 dataset, verifying its modularity performance.

\begin{table}[t]
\caption{Classification accuracy ($\%$) comparisons on CIFAR-100 dataset under NIID-1 scenario. FedDyn with Disco is highlighted for \textbf{best} performance. Experiments show our proposed FedDisco consistently improves the performances of baselines on CIFAR-100.}
\label{table:cifar100}
\vskip 0.15in
\begin{center}
\begin{small}
\begin{sc}
\begin{tabular}{ccccccccc}
\toprule
Method & FedAvg & FedAvgM & FedProx & SCAFFOLD & FedDyn & FedNova & MOON & FedDC
\\ \midrule
Without Disco & 57.28 & 56.27 & 58.99 & 61.39 & 62.29 & 57.59 & 58.01 & 62.20 \\
With Disco & 58.28 & 56.73 & 59.29 & 61.90 & \textbf{62.83} & 58.02 & 58.16 & 62.61\\ \bottomrule
\end{tabular}
\end{sc}
\end{small}
\end{center}
\vskip -0.2in
\end{table}

\subsection{Effects on Client-Level Fairness}
\label{append:fairness}

In this paper, we focus on the generalization ability and promote category-level fairness, where the corresponding global distribution is uniform. Client-level fairness is another interesting and important research topic in FL~\cite{lifair}. Here, we explore from this aspect for more comprehensive understanding.

Note that FedDisco does not hurt client-level fairness for two reasons. First, when a client's aggregation weight is high, this client has a uniform training distribution, which can naturally benefit all the clients more or less. If we do the opposite thing and assign a high aggregation weight to a client whose training category distribution is highly skewed, this only benefits this single client, hurts the other clients and violates the client-level fairness. Second, each client's test data distribution could be close to uniform distribution, even though its training data distribution is far from uniform. One important aim of federated learning is to enable clients with biased and limited training data to be aware of global distribution.

To validate that FedDisco does not hurt client-level fairness, we run $40$ rounds of FL on CIFAR-10 and record the variance of test accuracy across clients. This accuracy variance is often used for evaluation of fairness~\cite{lifair}. Small variance denotes that the test accuracies of clients are similar, indicating high fairness. Thus, a lower variance denotes higher fairness. \cref{table:client_fair} compares the accuracy variance of FedAvg and the accuracy variance after applying FedDisco to FedAvg. From the table, we see that the variance of FedDisco is comparable or lower than FedAvg, indicating that FedDisco can potentially enhance the client-level fairness.

\begin{table}[t]
\caption{Comparison of accuracy variance across clients at different round. Lower denotes more fair.}
\label{table:client_fair}
\vskip 0.15in
\begin{center}
\begin{small}
\begin{sc}
\begin{tabular}{cccccc}
\toprule
Round & 1 & 10 & 20 & 30 & 40 \\ \midrule
FedAvg & 178.05 & 226.38 & \textbf{73.65} & 94.49 & 89.97 \\
\textbf{FedDisco} & \textbf{167.40} & \textbf{110.63} & 74.70 & \textbf{84.72} & \textbf{73.45}\\ \bottomrule
\end{tabular}
\end{sc}
\end{small}
\end{center}
\vskip -0.2in
\end{table}

\subsection{Comparison with Equal Aggregation Weights}

Following the experiments in~\cref{table:main}, \cref{table:1/K} compares FedDisco with FedAvg with equal aggregation weights (i.e., $p_k=1/K$). From the table, we see that FedDisco performs significantly better across different settings.

\begin{table*}[t]
\caption{Comparison between FedDisco and FedAvg with equal aggregation weights $p_k=1/K$. FedDisco performs significantly better across different datasets and heterogeneity types.}
\label{table:1/K}
\vskip 0.15in
\begin{center}
\begin{small}
\begin{sc}
\begin{tabular}{cccccccc}
\toprule
\multirow{2}{*}{Method} & HAM10000 & \multicolumn{2}{c}{CIFAR-10} & \multicolumn{2}{c}{CINIC-10} & \multicolumn{2}{c}{Fashion-MNIST} \\
& NIID-1 & NIID-1  & NIID-2 & NIID-1 & NIID-2 & NIID-1& NIID-2 \\
\midrule
FedAvg ($p_k=1/K$)  & 44.76 & 68.11 & 65.60 & 54.27 & 50.35 & 89.35 & 86.46 \\
\textbf{FedAvg+Disco}  & \textbf{50.95} & \textbf{70.05} & \textbf{68.30} & \textbf{54.81} & \textbf{52.46} & \textbf{89.56} & \textbf{87.56} \\
\bottomrule
\end{tabular}
\end{sc}
\end{small}
\end{center}
\vskip -0.2in
\end{table*}

\subsection{Experiments on Device Heterogeneity}

For device heterogeneity where there are stragglers, we conduct experiments with the NIID-2 setting on CIFAR-10 following the setting in~\cite{fednova}, where we uniformly sample iteration numbers for each client in the range of $50$ to $500$ for each round. Results show that FedAvg achieves $60.21\%$ while \textbf{FedDisco achieves $63.20\%$, indicating that FedDisco can still bring performance improvement under setting of stragglers.}

Device heterogeneity is an orthogonal issue to distribution heterogeneity. As these two issues can be concurrent in practice, FedDisco can still play a role in enhancing overall performance. More importantly, device heterogeneity can even exacerbate the effect of distribution heterogeneity. For example, if Client A has a smaller discrepancy level and Client B has a larger discrepancy level. When Client A is a straggler that performs fewer iterations while Client B performs much more iterations, it is more critical to enhance the aggregation weight of Client A, otherwise the FL system will be dominated by Client B and the global model will be biased. Further, we can combine algorithms (e.g., FedNova~\cite{fednova}) that specifically focuses on device heterogeneity with ours to further enhance performance.

\subsection{Other Experiments}

1) The idea of discrepancy-aware collaboration can be extended to otehr tasks such as regression. Here, we conduct experiments on regression dataset Prediction of Facebook Comment~\cite{facebook} (predicting the number of comments given a post). We pre-define several categories for this regression task, where each category covers regression labels in some specific range. For example, we group those samples with regression labels ranging from 1 to 10 as category 1. With this strategy, we can obtain a distribution vector as we do for classification tasks. Note that such operation is only conducted for obtaining a distribution vector, we still use the regression labels for model training. It turns out that the test loss of FedAvg~\cite{fedavg} is 0.432 and FedDisco achieves 0.407, indicating our FedDisco can also handle continuous label distributions.

2) Here, we consider feature-level heterogeneity. We conduct experiments on FEMNIST from LEAF benchmark~\cite{leaf} following~\cite{fedprox}, where both feature-level and category-level heterogeneity exist. We see that FedAvg achieves $76.40\%$ accuracy while FedDisco achieves $78.36\%$ accuracy.

\section{Theoretical Analysis}

\subsection{Preliminaries}

The global objective function is $ F (\mathbf{w})=\sum_{k=1}^N p_k F_k(\mathbf{w})$, where $\sum_{k=1}^K p_k = 1$. For ease of writing, we use $g_k(\mathbf{w})$ to denote mini-batch gradient $g_k(\mathbf{w}|\xi)$ and $\nabla F_k(\mathbf{w})$ to denote full-batch gradient, where $\xi$ is a mini-batch sampled from dataset. We further define the following two notions:

\begin{equation}
\text{Averaged Mini-batch Gradient:} \quad \mathbf{d}_k=\frac{1}{\tau} \sum_{r=0}^{\tau-1} g_k(\mathbf{w}_k^{(t,r)}),
\end{equation}

\begin{equation}
\text{Averaged Full-batch Gradient:} \quad \mathbf{h}_k=\frac{1}{\tau} \sum_{r=0}^{\tau-1} \nabla F _k(\mathbf{w}_k^{(t,r)}).
\end{equation}

Then, the update of the global model between two rounds is as follows:
\begin{equation}
\mathbf{w}^{(t+1,0)}-\mathbf{w}^{(t,0)}=-\tau \eta \sum_{k=1}^K p_k \mathbf{d}_k.
\end{equation}

Here, we presents a key lemma and defer its proof to section~\ref{sct:proof_lemma}.

\begin{lemma}
\label{lemma_1}
Suppose $\{ A_t \}_{t=1}^T$ is a sequence of random matrices and follows $\mathbb{E}[A_t|A_{t-1},A_{t-2},...,A_1]=\mathbf{0}$, then
\begin{equation*}
     \mathbb{E} \left[ \left\| \sum_{t=1}^T A_t \right\|^2_F \right] =\sum_{t=1}^T \mathbb{E} \left[ \left\| A_t \right\|^2_F \right]
\end{equation*}
\end{lemma}

\subsection{Proof of~\cref{theorem_1}}
\label{append_proof_theorem}

According to the Lipschitz-smooth assumption in~\cref{ass_smooth}, we have its equivalent form~\cite{l_smooth}

\begin{align}
    & \quad \mathbb{E} \left[ F (\mathbf{w}^{(t+1,0)}) \right] - F (\mathbf{w}^{(t,0)})\notag \\
    &\leq \mathbb{E} \left[ \left \langle \nabla F (\mathbf{w}^{(t,0)}), \mathbf{w}^{(t+1,0)}-\mathbf{w}^{(t,0)} \right \rangle \right] - \frac{L}{2} \mathbb{E} \left[ \left\| \mathbf{w}^{(t+1,0)} - \mathbf{w}^{(t,0)}  \right\|^2 \right]\\
    &= -\tau \eta \underbrace{\mathbb{E} \left[ \left\langle \nabla F (\mathbf{w}^{(t,0)}), \sum_{k=1}^Kp_k\mathbf{d}_k \right\rangle \right]}_{N_1} + \frac{L\tau^2\eta^2}{2} \underbrace{\mathbb{E} \left[ \left\| \sum_{k=1}^K p_k \mathbf{d}_k \right\|^2 \right]}_{N_2}, \label{T1T2}
\end{align}
where the expectation is taken over mini-batches $\xi_k^{(t,r)}$, $\forall k \in {1,2,...,K}$, $r \in {0,1,...,\tau-1}$.

\subsubsection{Bounding $N_1$ in (\ref{T1T2})}

\begin{align}
    N_1 &= \mathbb{E} \left[ \left\langle \nabla F (\mathbf{w}^{(t,0)}), \sum_{k=1}^K p_k (\mathbf{d}_k-\mathbf{h}_k) \right\rangle \right] + \mathbb{E} \left[ \left\langle \nabla F (\mathbf{w}^{(t,0)}), \sum_{k=1}^K p_k \mathbf{h}_k \right\rangle \right] \\
    & = \mathbb{E} \left[ \left\langle \nabla F (\mathbf{w}^{(t,0)}), \sum_{k=1}^K p_k \mathbf{h}_k \right\rangle \right] \label{T1_2} \\
    & = \frac{1}{2} \left\| \nabla F(\mathbf{w}^{(t,0)}) \right\|^2 + 
        \frac{1}{2} \mathbb{E} \left[ \left\| \sum_{k=1}^K p_k \mathbf{h}_k \right\|^2 \right]
        - \frac{1}{2} \mathbb{E} \left[ \left\| \nabla F(\mathbf{w}^{(t,0)}) - \sum_{k=1}^K p_k \mathbf{h}_k \right\|^2 \right], \label{T1_last}
\end{align}
where (\ref{T1_2}) uses the unbiased gradient assumption in~\cref{ass_grad}, such that $\mathbb{E}[\mathbf{d}_k-\mathbf{h}_k]=\mathbf{h}_k-\mathbf{h}_k=\mathbf{0}$. (\ref{T1_last}) uses the fact that $2\left\langle a,b \right\rangle = \left\| a \right\|^2 + \left\| b \right\|^2 - \left\| a-b \right\|^2$.

\subsubsection{Bounding $N_2$ in (\ref{T1T2})}

\begin{align}
    N_2 &= \mathbb{E} \left[ \left\| \sum_{k=1}^K p_k (\mathbf{d}_k-\mathbf{h}_k) + \sum_{k=1}^K p_k \mathbf{h}_k \right\|^2 \right] \\
    & \leq 2 \mathbb{E} \left[ \left\| \sum_{k=1}^K p_k (\mathbf{d}_k-\mathbf{h}_k) \right\|^2 \right]
        + 2 \mathbb{E} \left[ \left\| \sum_{k=1}^K p_k \mathbf{h}_k \right\|^2 \right] \label{T2_2} \\
    & = 2 \sum_{k=1}^K p_k^2 \mathbb{E}  \left[ \left\| \mathbf{d}_k-\mathbf{h}_k \right\|^2 \right]
        + 2 \mathbb{E} \left[ \left\| \sum_{k=1}^K p_k \mathbf{h}_k \right\|^2 \right] \label{T2_3} \\
    & = \frac{2}{\tau^2} \sum_{k=1}^K p_k^2 \mathbb{E}  \left[ \left\| \sum_{r=0}^{\tau-1} (g_k(\mathbf{w}_k^{(t,r)})-\nabla F_k(\mathbf{w}_k^{(t,r)})) \right\|^2 \right]
        + 2 \mathbb{E} \left[ \left\| \sum_{k=1}^K p_k \mathbf{h}_k \right\|^2 \right] \\
    & = \frac{2}{\tau^2} \sum_{k=1}^K p_k^2 \sum_{r=0}^{\tau-1} \mathbb{E}  \left[ \left\| g_k(\mathbf{w}_k^{(t,r)})-\nabla F_k(\mathbf{w}_k^{(t,r)}) \right\|^2 \right]
        + 2 \mathbb{E} \left[ \left\| \sum_{k=1}^K p_k \mathbf{h}_k \right\|^2 \right] \label{T2_5} \\
    & \leq \frac{2\sigma^2}{\tau} \sum_{k=1}^K p_k^2 
        + 2 \mathbb{E} \left[ \left\| \sum_{k=1}^K p_k \mathbf{h}_k \right\|^2 \right] \label{T2_last}
\end{align}
where (\ref{T2_2}) follows $\left\| a+b \right\|^2 \leq 2\left\| a \right\|^2 + 2\left\| b \right\|^2$, (\ref{T2_3}) uses the fact that clients are independent to each other so that $\mathbb{E} \left\langle \mathbf{d}_k-\mathbf{h}_k, \mathbf{d}_n-\mathbf{h}_n \right\rangle=0, \forall k \neq n$. (\ref{T2_5}) uses~\cref{lemma_1} and (\ref{T2_last}) uses bounded variance assumption in~\cref{ass_grad}.

Plug (\ref{T1_last}) and (\ref{T2_last}) back into (\ref{T1T2}), we have
\begin{align}
    & \mathbb{E} \left[ F (\mathbf{w}^{(t+1,0)}) \right] - F (\mathbf{w}^{(t,0)})\notag \\
    \leq & - \frac{\tau \eta}{2} \left\| \nabla F(\mathbf{w}^{(t,0)}) \right\|^2 
        - \frac{\tau \eta}{2} (1-2\tau \eta L) \mathbb{E} \left[ \left\| \sum_{k=1}^K p_k \mathbf{h}_k \right\|^2 \right] +L \tau \eta^2 \sigma^2 \sum_{k=1}^K p_k^2
        + \frac{\tau \eta}{2} \underbrace{\mathbb{E} \left[ \left\| \nabla F(\mathbf{w}^{(t,0)}) - \sum_{k=1}^K p_k \mathbf{h}_k \right\|^2 \right]}_{N_3} \label{N3_first}.
\end{align}

\subsubsection{Bounding $N_3$ in (\ref{N3_first})}

\begin{align}
    & \mathbb{E} \left[ \left\| \nabla F(\mathbf{w}^{(t,0)}) - \sum_{k=1}^K p_k \mathbf{h}_k \right\|^2 \right] \notag \\
    = & \mathbb{E} \left[ \left\| \sum_{k=1}^K (n_k-p_k)\nabla F_k(\mathbf{w}^{(t,0)}) + \sum_{k=1}^K p_k\left(\nabla F_k(\mathbf{w}^{(t,0)}) - \mathbf{h}_k \right)  \right\|^2 \right] \\
    \leq & 2 \left\| \sum_{k=1}^K (n_k-p_k)\nabla F_k(\mathbf{w}^{(t,0)}) \right\|^2 + 2 \left\| \sum_{k=1}^K p_k\left(\nabla F_k(\mathbf{w}^{(t,0)}) - \mathbf{h}_k\right)  \right\|^2 \label{N3_1}\\
    \leq & 2  \left[ \sum_{k=1}^K (n_k-p_k)^2 \right] \left[ \sum_{k=1}^K \left\|\nabla F_k(\mathbf{w}^{(t,0)}) \right\|^2 \right] + 2 \left\| \sum_{k=1}^K p_k\left(\nabla F_k(\mathbf{w}^{(t,0)}) - \mathbf{h}_k\right)  \right\|^2 \label{N3_2}\\
    \leq & 2 \left[ \sum_{k=1}^K (n_k-p_k)^2 \right] \left[ K \left\| \nabla F(\mathbf{w}^{(t,0)}) \right\|^2 + B \sum_{k=1}^K d_k \right] + 2 \left\| \sum_{k=1}^K p_k\left(\nabla F_k(\mathbf{w}^{(t,0)}) - \mathbf{h}_k\right)  \right\|^2 \label{N3_3},
\end{align}
where (\ref{N3_1}) follows $\left\| a+b \right\|^2 \leq 2\left\| a \right\|^2 + 2\left\| b \right\|^2$, (\ref{N3_2}) follows Cauchy–Schwarz inequality, (\ref{N3_3}) uses the bounded similarity assumption in~\cref{ass_diss}.

We use $W_D$ to denote $2K \left[ \sum_{k=1}^K (n_k-p_k)^2 \right]$. When $1-2\tau \eta L \geq 0$, we have
\begin{align}
    & \quad \mathbb{E} \left[ F (\mathbf{w}^{(t+1,0)}) \right] - F (\mathbf{w}^{(t,0)})\notag \\
    & \leq - \frac{\tau \eta (1-W_D)}{2} \left\| \nabla F(\mathbf{w}^{(t,0)}) \right\|^2 
        +L \tau \eta^2 \sigma^2 \sum_{k=1}^K p_k^2 + \frac{\tau \eta W_D B}{2K} \sum_{k=1}^K d_k
        + \tau \eta \mathbb{E} \left[ \left\| \sum_{k=1}^K p_k \left(\nabla F_k(\mathbf{w}^{(t,0)}) -  \mathbf{h}_k\right) \right\|^2 \right] \\
    & \leq - \frac{\tau \eta (1-W_D)}{2} \left\| \nabla F(\mathbf{w}^{(t,0)}) \right\|^2 
        +L \tau \eta^2 \sigma^2 \sum_{k=1}^K p_k^2 + \frac{\tau \eta W_D B}{2K} \sum_{k=1}^K d_k
        + \tau \eta \sum_{k=1}^K p_k \underbrace{\mathbb{E} \left[ \left\| \nabla F_k(\mathbf{w}^{(t,0)}) -  \mathbf{h}_k \right\|^2 \right]}_{N_4} \label{N4},
\end{align}
where (\ref{N4}) uses Jensen's Inequality $\left\| \sum_{k=1}^K p_k x_k \right\|^2 \leq \sum_{k=1}^K p_k \left\| x_k \right\|^2$.

\subsubsection{Bounding $N_4$ in (\ref{N4})}
\begin{align}
    \mathbb{E} \left[ \left\| \nabla F_k(\mathbf{w}^{(t,0)}) - \mathbf{h}_k \right\|^2 \right]
        & = \mathbb{E} \left[ \left\| \nabla F_k(\mathbf{w}^{(t,0)}) - \frac{1}{\tau} \sum_{r=0}^{\tau-1} 
            \nabla F_k(\mathbf{w}_k^{(t,r)}) \right\|^2 \right] \\
        & = \mathbb{E} \left[ \left\| \frac{1}{\tau} \sum_{r=0}^{\tau-1} (\nabla F_k(\mathbf{w}^{(t,0)}) - 
            \nabla F_k(\mathbf{w}_k^{(t,r)})) \right\|^2 \right] \\
        & \leq \frac{1}{\tau} \sum_{r=0}^{\tau-1} \mathbb{E} \left[ \left\| 
            \nabla F_k(\mathbf{w}^{(t,0)}) - \nabla F_k(\mathbf{w}_k^{(t,r)}) \right\|^2 \right] \label{N4_3} \\
        & \leq \frac{L^2}{\tau} \sum_{r=0}^{\tau-1} \underbrace{\mathbb{E} \left[ \left\| 
            \mathbf{w}^{(t,0)} - \mathbf{w}_k^{(t,r)} \right\|^2 \right]}_{N_5} \label{N4_last},
\end{align}
where (\ref{N4_3}) uses Jensen's Inequality and (\ref{N4_last}) follows Lipschitz-smooth property.

\subsubsection{Bounding $N_5$ in (\ref{N5_last})}
\begin{align}
    &\mathbb{E} \left[ \left\| \mathbf{w}^{(t,0)} - \mathbf{w}_k^{(t,r)} \right\|^2 \right]
     = \eta^2 \mathbb{E} \left[ \left\| \sum_{s=0}^{r-1} g_k(\mathbf{w}_k^{(t,s)}) \right\|^2 \right] \\
     \leq & 2\eta^2 \mathbb{E} \left[ \left\| \sum_{s=0}^{r-1} \left( g_k(\mathbf{w}_k^{(t,s)})-\nabla F_k(\mathbf{w}_k^{(t,s)}) \right) \right\|^2 \right]
        + 2\eta^2 \mathbb{E} \left[ \left\| \sum_{s=0}^{r-1}  \nabla F_k(\mathbf{w}_k^{(t,s)}) \right\|^2 \right] \label{N5_2}\\
    = & 2\eta^2 \sum_{s=0}^{r-1} \mathbb{E} \left[ \left\|  g_k(\mathbf{w}_k^{(t,s)})-\nabla F_k(\mathbf{w}_k^{(t,s)}) \right\|^2 \right]
        + 2\eta^2 \mathbb{E} \left[ \left\| \sum_{s=0}^{r-1}  \nabla F_k(\mathbf{w}_k^{(t,s)}) \right\|^2 \right] \label{N5_3}\\
    \leq & 2r\eta^2\sigma^2 + 2\eta^2 \mathbb{E} \left[ \left\| r \sum_{s=0}^{r-1} \frac{1}{r}  \nabla F_k(\mathbf{w}_k^{(t,s)}) \right\|^2 \right] \label{N5_4}\\
    \leq & 2r\eta^2\sigma^2 + 2r\eta^2 \sum_{s=0}^{r-1} \mathbb{E} \left[ \left\| \nabla F_k(\mathbf{w}_k^{(t,s)}) \right\|^2 \right] \label{N5_5}\\
    \leq & 2r\eta^2\sigma^2 + 2r\eta^2 \sum_{s=0}^{\tau-1} \mathbb{E} \left[ \left\| \nabla F_k(\mathbf{w}_k^{(t,s)}) \right\|^2 \right] \label{N5_last}
\end{align}
where (\ref{N5_2}) uses $\left\| a+b \right\|^2 \leq  2 \left\| a \right\|^2 + 2 \left\| b \right\|^2$, (\ref{N5_3}) uses~\cref{lemma_1}, (\ref{N5_4}) uses the bounded variance assumption in~\cref{ass_grad}, (\ref{N5_5}) uses Jensen's Inequality.

Plug (\ref{N5_last}) back into (\ref{N4_last}) and use this equation $\sum_{r=0}^{\tau-1}r=\frac{\tau(\tau-1)}{2}$, we have
\begin{align}
    &\mathbb{E} \left[ \left\| \nabla F_k(\mathbf{w}^{(t,0)}) - \mathbf{h}_k \right\|^2 \right] \leq \frac{L^2}{\tau} \sum_{r=0}^{\tau-1} \mathbb{E} \left[ \left\| \mathbf{w}^{(t,0)} - \mathbf{w}_k^{(t,r)} \right\|^2 \right] \\
    \leq & (\tau-1)L^2\eta^2\sigma^2+(\tau-1)L^2\eta^2\sum_{s=0}^{\tau-1} \underbrace{\mathbb{E} \left[ \left\| \nabla F_k(\mathbf{w}_k^{(t,s)}) \right\|^2 \right]}_{N_6} \label{N6},
\end{align}
where $N_6$ in (\ref{N6}) can be further bounded.

\subsubsection{Bounding $N_6$ in (\ref{N6})}
\begin{align}
    & \mathbb{E} \left[ \left\| \nabla F_k(\mathbf{w}_k^{(t,s)}) \right\|^2 \right] \notag \\
    \leq & 2 \mathbb{E} \left[ \left\| \nabla F_k(\mathbf{w}_k^{(t,s)})-\nabla F_k(\mathbf{w}^{(t,0)}) \right\|^2 \right]
        + 2 \mathbb{E} \left[ \left\| \nabla F_k(\mathbf{w}^{(t,0)}) \right\|^2 \right] \label{N6_2}\\
    \leq & 2 L^2 \mathbb{E} \left[ \left\| \mathbf{w}^{(t,0)}-\mathbf{w}_k^{(t,s)} \right\|^2 \right]
        + 2 \mathbb{E} \left[ \left\| \nabla F_k(\mathbf{w}^{(t,0)}) \right\|^2 \right] \label{N6_last},
\end{align}
where (\ref{N6_2}) uses $\left\| a+b \right\|^2 \leq  2 \left\| a \right\|^2 + 2 \left\| b \right\|^2$, (\ref{N6_last}) uses Lipschitz-smooth property. Plug (\ref{N6_last}) back to (\ref{N6}), we have
\begin{align}
    & \frac{L^2}{\tau} \sum_{r=0}^{\tau-1} \mathbb{E} \left[ \left\| \mathbf{w}^{(t,0)} - \mathbf{w}_k^{(t,r)} \right\|^2 \right]\notag \\
    \leq & (\tau-1)L^2\eta^2\sigma^2+2(\tau-1)\eta^2L^4\sum_{s=0}^{\tau-1} \mathbb{E} \left[ \left\| \mathbf{w}_k^{(t,0)}-\mathbf{w}^{(t,s)} \right\|^2 \right] \notag \\
        & \quad + 2(\tau-1)\eta^2L^2 \sum_{s=0}^{\tau-1} \mathbb{E} \left[ \left\| \nabla F_k(\mathbf{w}^{(t,0)}) \right\|^2 \right] 
\end{align}

After rearranging, we have

\begin{align}
    &\mathbb{E} \left[ \left\| \nabla F_k(\mathbf{w}^{(t,0)}) - \mathbf{h}_k \right\|^2 \right] \notag \\
    \leq & \frac{L^2}{\tau} \sum_{r=0}^{\tau-1} \mathbb{E} \left[ \left\| \mathbf{w}^{(t,0)} - \mathbf{w}_k^{(t,r)} \right\|^2 \right] \\
    \leq & \frac{(\tau-1)\eta^2\sigma^2L^2}{1-2\tau(\tau-1)\eta^2L^2} + \frac{2\tau(\tau-1)\eta^2L^2}{1-2\tau(\tau-1)\eta^2L^2} \mathbb{E} \left[ \left\| \nabla F_k(\mathbf{w}^{(t,0)}) \right\|^2 \right] \\
    = & \frac{(\tau-1)\eta^2\sigma^2L^2}{1-A} + \frac{A}{1-A} \mathbb{E} \left[ \left\| \nabla F_k(\mathbf{w}^{(t,0)}) \right\|^2 \right] \label{T3_rearange},
\end{align}
where we define $A=2\tau(\tau-1)\eta^2L^2 < 1$. Then, the last term in (\ref{N4}) is bounded by

\begin{align}
    &\tau\eta \sum_{k=1}^K p_k \mathbb{E} \left[ \left\| \nabla F_k(\mathbf{w}^{(t,0)}) - \mathbf{h}_k \right\|^2 \right] \notag \\
    \leq & \tau\eta \sum_{k=1}^K \left\{ p_k \left[ \frac{(\tau-1)\eta^2\sigma^2L^2}{1-A} + \frac{A}{1-A} \mathbb{E} \left[ \left\| \nabla F_k(\mathbf{w}^{(t,0)}) \right\|^2 \right] \right] \right\} \\
    \leq & \frac{\tau(\tau-1)\sigma^2L^2\eta^3}{1-A} + \frac{\tau \eta A}{1-A} \mathbb{E} \left[ \left\| \nabla F (\mathbf{w}^{(t,0)}) \right\|^2 \right] + \frac{\tau \eta A B}{1-A} \sum_{k=1}^K p_k d_k \label{plug_to_T3},
\end{align}
where (\ref{plug_to_T3}) follows bounded dissimilarity assumption in~\cref{ass_diss}. Plug (\ref{plug_to_T3}) back to (\ref{N4}), we have
\begin{align}
    &\mathbb{E} \left[ F (\mathbf{w}^{(t+1,0)}) \right] - F (\mathbf{w}^{(t,0)})\notag \\
    \leq & - \frac{\tau \eta (1-W_D)}{2} \left\| \nabla F(\mathbf{w}^{(t,0)}) \right\|^2 +L \tau \eta^2 \sigma^2 \sum_{k=1}^K p_k^2 + \frac{\tau \eta W_D B}{2K} \sum_{k=1}^K d_k \notag\\
    & +\frac{\tau(\tau-1)\sigma^2L^2\eta^3}{1-A} + \frac{\tau \eta A}{1-A} \mathbb{E} \left[ \left\| \nabla F (\mathbf{w}^{(t,0)}) \right\|^2 \right] + \frac{\tau \eta A B}{1-A} \sum_{k=1}^K p_k d_k \\
    = & - \frac{\tau \eta}{2} (1-W_D-\frac{2A}{1-A}) \left\| \nabla F(\mathbf{w}^{(t,0)}) \right\|^2 + L \tau \eta^2 \sigma^2 \sum_{k=1}^K p_k^2 + \frac{\tau \eta W_D B}{2K} \sum_{k=1}^K d_k + \frac{\tau(\tau-1)\sigma^2L^2\eta^3}{1-A} + \frac{\tau \eta A B}{1-A} \sum_{k=1}^K p_k d_k
\end{align}

Finally, by taking the average expectation across all rounds, we finish the proof of~\cref{theorem_1}
\begin{align}
    & \mathop{\min}_{t} \mathbb{E} \left\| \nabla F(\mathbf{w}^{(t,0)}) \right\|^2 \leq \frac{1}{T} \sum_{t=0}^{T-1} \mathbb{E} \left\| \nabla F(\mathbf{w}^{(t,0)}) \right\|^2 \\
    \leq & \frac{2(1-A)\left( F(\mathbf{w}^{(0,0)}) - F_{inf} \right)}{\tau \eta T \left[ 1-3A-W_D(1-A) \right]} + \frac{(1-A)W_DB\sum_{k=1}^Kd_k}{\left[ 1-3A-W_D(1-A) \right]K} \notag\\
    & + \frac{2(1-A)L\eta \sigma^2 \sum_{k=1}^Kp_k^2}{\left[ 1-3A-W_D(1-A) \right]} + \frac{2(\tau-1)\sigma^2L^2\eta^2}{\left[ 1-3A-W_D(1-A) \right]} + \frac{2AB\sum_{k=1}^Kp_kd_k}{\left[ 1-3A-W_D(1-A) \right]} \\
    = & \frac{1}{1-3A-W_D(1-A)} \bigg( \underbrace{\frac{2(1-A)\left( F(\mathbf{w}^{(0,0)}) - F_{inf} \right)}{\tau \eta T}}_{T_1} +  \underbrace{\frac{(1-A)W_DB}{K} \sum_{k=1}^Kd_k}_{T_2} \notag\\
    &+ \underbrace{2(1-A)L\eta \sigma^2 \sum_{k=1}^Kp_k^2}_{T_3} + \underbrace{2(\tau-1)\sigma^2L^2\eta^2}_{T_4} + \underbrace{2AB\sum_{k=1}^Kp_kd_k}_{T_5} \bigg) \label{theorem_last},
\end{align}
where $W_D=2K \left[ \sum_{k=1}^K (n_k-p_k)^2 \right]$, $p_k$ is the aggregation weight, $n_k$ is the dataset relative size, $d_k$ is the discrepancy level, $A=2\tau(\tau-1)\eta^2L^2 < 1$, $\tau$ is the number of steps in local model training, $\eta$ is learning rate, $T$ is the total communication round in FL, $K$ is the total client number, $F_{inf}, B, L, \sigma$ are the constants in assumptions.

\subsection{Proof of~\cref{lemma_1}}
\label{sct:proof_lemma}

\emph{Suppose $\{ A_t \}_{t=1}^T$ is a sequence of random matrices and follows $\mathbb{E}[A_t|A_{t-1},A_{t-2},...,A_1]=\mathbf{0}$, then}
\begin{equation*}
    \mathbb{E} \left[ \left\| \sum_{t=1}^T A_t \right\|^2_F \right] =\sum_{t=1}^T \mathbb{E} \left[ \left\| A_t \right\|^2_F \right]
\end{equation*}
\emph{Proof.}
\begin{align}
    \mathbb{E} \left[ \left\| \sum_{t=1}^T A_t \right\|^2_F \right] 
    &=\sum_{t=1}^T \mathbb{E} \left[ \left\| A_t \right\|^2_F \right] + \sum_{i=1}^T\sum_{j=1,j\neq i}^T \mathbb{E} \left[ Tr \left\{ A_i^\top A_j^\top \right\} \right]\\
    &=\sum_{t=1}^T \mathbb{E} \left[ \left\| A_t \right\|^2_F \right] + \sum_{i=1}^T\sum_{j=1,j\neq i}^T Tr \left\{ \mathbb{E} \left[ A_i^\top A_j^\top  \right] \right\}\\
    &=\sum_{t=1}^T \mathbb{E} \left[ \left\| A_t \right\|^2_F \right] \label{lemma1_last},
\end{align}
where (\ref{lemma1_last}) comes from assuming $i<j$ and using the law of total expectation $\mathbb{E} \left[ A_i^\top A_j \right]=\mathbb{E} \left[ A_i^\top \mathbb{E}[A_j|A_{i},...,A_1] \right]=\mathbf{0}$.

\subsection{Further Analysis with Proper Learning Rate}

As a conventional setting in theoretical literature, we can set the learning rate $\eta=\frac{1}{\sqrt{\tau T}}$~\cite{fednova}. Here, note that the learning rate $\eta$ is strongly correlated with the number of communication round $T$. Then, there are two typical cases:

\textbf{$T$ is finite and relatively small.} According to $\eta=\frac{1}{\sqrt{\tau T}}$, the learning rate $\eta$ will be relatively large. Then, $A=2\tau(\tau-1)\eta^2L^2$ will be relatively large and $(1-A)$ will be relatively small. As a result, $T_2$ will be relatively small and $T_5$ will be relatively large, making $T_5$ a dominant term in the optimization bound. Thus, in this case, tuning $p_k$ to make $T_5$ smaller is a better solution such that the $a$ and $b$ in~\cref{eq:disco_weight} will be non-zero and the overall upper bound could be tighter.

\textbf{$T$ is quite large or infinite.} the learning rate $\eta=\frac{1}{\sqrt{\tau T}}$ will be relatively small. Then, $A=2\tau(\tau-1)\eta^2L^2$ will be relatively small and $(1-A)$ will be relatively large. As a result, $T_2$ will be relatively large and $T_5$ will be relatively small, making $T_2$ a dominant term in the optimization bound. Thus, in this case, tuning $p_k$ to make $T_2$ smaller is a better solution. To achieve this, $W_D$ should be reduced to zero and thus $a$ and $b$ in~\cref{eq:disco_weight} will zero to make the upper bound tight. Here, we can have the following corollary:
\begin{corollary}
\label{corollary_1}
By substituting $\eta=\frac{1}{\sqrt{\tau T}}$ into~\cref{theorem_1}, we will have the following bound:
\begin{equation}
    \mathop{\min}_{t} \mathbb{E} ||\nabla F (\mathbf{w}^{(t,0)})||^2 \leq \mathcal{O} (\frac{1}{\sqrt{\tau T}}) + \mathcal{O} (\frac{1}{T}) + \mathcal{O} (\frac{\tau}{T}).
\end{equation}
\end{corollary}

\cref{corollary_1} indicates that as $T \rightarrow \infty$, the optimization upper bound approaches $0$ (reaches a stationary point) and matches with previous convergence results~\cite{fednova}.

As the ultimate goal of this paper is to achieve more pleasant performance in practice, we are more interested in the previous part that the number of communication round $T$ is finite ($T$ should not be too large due to issues such as communication burden). For this reason, we explore better aggregation weights to achieve a tighter upper bound in this paper. However, our algorithm can actually converge (reach a stationary point) based on our~\cref{theorem_1} and \cref{corollary_1} if the number of communication round goes to infinite.

\subsection{Reformulated Upper Bound Minimization}
\label{append_proof_obtain_pk}

Generally, a tighter bound corresponds to a better optimization result. Thus, we explore the effects of $p_k$ on the upper bound in (\ref{theorem_last}). In~\cref{theorem_1}, there there are four parts related to $p_k$. First, we see that larger difference between $p_k$ and $n_k$ contributes to larger $W_D$ and thus smaller denominator in $T_0$ and larger value in $T_2$, which tends to loose the bound. As for $T_5$, by setting $p_k$ negatively correlated to $d_k$, when clients have different level of discrepancy, i.e. different $d_k$, $T_5$ will be further reduced, which tends to tighten the bound. Therefore, there could be an optimal set of $\{p_k | k \in [K]\}$ that contributes to the tightest bound, where an optimal $p_k$ should be correlated to both $n_k$ and $d_k$.

To minimize this upper bound, directly solving the minimization results in a complicated expression, which involves too many unknown hyper-parameters in practice. To simplify the expression, we convert the original objective from minimizing $(T_1+T_2+T_3+T_4+T_5)/T_0$ to minimizing $T_1+T_2+T_3+T_4+T_5 - \lambda T_0$, where $\lambda$ is a hyper-parameter. The converted objective still promotes maximization of $T_0$ and minimization of $T_1+T_2+T_3+T_4+T_5$, and still contributes to tighten the bound $(T_1+T_2+T_3+T_4+T_5)/T_0$. Then, our discrepancy-aware aggregation weight is obtained through solving the following optimization problem:
\begin{equation*}
\begin{aligned}
\mathop{\min}_{\{p_k\}} & \frac{2(1-A)[F(\mathbf{w}^{(0,0)})-F_{inf}]}{\tau \eta T} + \frac{(1-A)W_DB}{K}\sum_{m=1}^Kd_m + 2(1-A)L\eta \sigma^2 \sum_{m=1}^K p_m^2\\
& + 2(\tau-1)\sigma^2 L^2 \eta^2 + 2AB\sum_{m=1}^K p_m d_m -
\lambda \left( 1-3A-W_D(1-A) \right), \\
{\rm s.t.} & \sum_{m} p_m =1,~p_m \geq 0,
\end{aligned}
\end{equation*}
where $W_D=2K \left[ \sum_{m=1}^K (n_m-p_m)^2 \right]$.

To solve this constrained optimization, one condition of the optimal solution is that the derivative of the following function equals to zero:
\begin{equation*}
\begin{aligned}
Q(p_k) = & \frac{2(1-A)[F(\mathbf{w}^{(0,0)})-F_{inf}]}{\tau \eta T} + \frac{(1-A)W_DB}{K}\sum_{m=1}^Kd_m + 2(1-A)L\eta \sigma^2 \sum_{m=1}^K p_m^2\\
& + 2(\tau-1)\sigma^2 L^2 \eta^2 + 2AB\sum_{m=1}^K p_m d_m -
\lambda \left( 1-3A-W_D(1-A) \right) + \mu (\sum_{m=1}^K p_m-1) - \sum_{m=1}^K \nu_m p_m
\end{aligned}
\end{equation*}

Then, we have the following equation:
\begin{equation*}
\begin{aligned}
4(1-A)B\sum_{m=1}^Kd_m (p_k-n_k) + 4(1-A)L\eta\sigma^2p_k + 2ABd_k + 4K\lambda(1-A)(p_k-n_k) + \mu - \nu_m = 0,
\end{aligned}
\end{equation*}
from which we can rearrange and obtain the expression of $p_k$:
\begin{equation}
\label{eq:pk_detail}
p_k = \frac{\left[ 4B(1-A)\sum_{m=1}^K d_m + 4K\lambda(1-A) \right]n_k - 2ABd_k - \mu + \nu_k}{4B(1-A) \sum_{m=1}^Kd_m + 4(1-A)L\eta \sigma^2 + 4K\lambda(1-A)}.
\end{equation}

Finally, we can derive the following concise expression of an optimized aggregation weight $p_k$:

\begin{equation}
    p_k \propto n_k - a \cdot d_k + b,
\end{equation}
where $a, b$ are two constants.

This theoretically show that simply dataset-size-based aggregation could be not optimal since the above analysis suggests that for a tighter upper bound, the aggregation weight $p_k$ should be correlated with both dataset size $n_k$ and local discrepancy level $d_k$. This expression of aggregation weight can mitigate the limitation of standard dataset size weighted aggregation by assigning larger aggregation weight to client with larger dataset size and smaller discrepancy level.

%%%%%%%%%%%%%%%%%%%%%%%%%%%%%%%%%%%%%%%%%%%%%%%%%%%%%%%%%%%%%%%%%%%%%%%%%%%%%%%
%%%%%%%%%%%%%%%%%%%%%%%%%%%%%%%%%%%%%%%%%%%%%%%%%%%%%%%%%%%%%%%%%%%%%%%%%%%%%%%

\end{document}